\documentclass[twoside]{article}
\usepackage[utf8]{inputenc}
\usepackage{hyperref}
\usepackage{amsmath}
\usepackage{amssymb}
\usepackage{amsfonts}
\usepackage{xfrac}
\usepackage{amsthm}
\usepackage{thmtools,thm-restate}

\newtheorem{thm}{Theorem}
\newtheorem{defn}{Definition}
\newtheorem{lem}{Lemma}

\usepackage[numbers]{natbib}
\usepackage{subcaption}

\usepackage{chngcntr}
\usepackage{apptools}
\AtAppendix{\counterwithin{lem}{section}}
\AtAppendix{\counterwithin{thm}{section}}
\AtAppendix{\counterwithin{defn}{section}}
\AtAppendix{\counterwithin{prop}{section}}

\makeatletter
\long\def\@makecaption#1#2{
  \vskip 0.8ex
  \setbox\@tempboxa\hbox{\small {\bf #1:} #2}
  \parindent 1.5em  
  \dimen0=\hsize
  \advance\dimen0 by -3em
  \ifdim \wd\@tempboxa >\dimen0
  \hbox to \hsize{
    \parindent 0em
    \hfil 
    \parbox{\dimen0}{\def\baselinestretch{0.96}\small
      {\bf #1.} #2
    } 
    \hfil}
  \else \hbox to \hsize{\hfil \box\@tempboxa \hfil}
  \fi
}
\makeatother

\usepackage{algorithm,algorithmic}

\usepackage{commath}
\usepackage{comment}
\usepackage{url}

\usepackage{macros}

\usepackage{todonotes}
\usepackage[accepted]{aistats2018}
\begin{document}

%

%

\twocolumn[

\aistatstitle{Derivative free optimization via repeated classification}

\aistatsauthor{ Tatsunori B. Hashimoto \And Steve Yadlowsky \And  John C. Duchi }

\vspace{1pt}
\aistatsaddress{Department of Statistics, Stanford University, Stanford, CA, 94305\\
  \texttt{\{thashim, syadlows, jduchi\}@stanford.edu}
}]

\begin{abstract}
  We develop an algorithm for minimizing a function using $n$ batched
  function value measurements at each of $T$ rounds by using classifiers to
  identify a function's sublevel set.  We show that sufficiently accurate
  classifiers can achieve linear convergence rates, and show that the
  convergence rate is tied to the difficulty of active learning sublevel
  sets. Further, we show that the bootstrap is a computationally efficient
  approximation to the necessary classification scheme.
  
  The end result is a computationally efficient derivative-free algorithm requiring no
  tuning that consistently outperforms other approaches on simulations,
  standard benchmarks, real-world DNA binding optimization, and airfoil
  design problems whenever batched function queries are natural.
\end{abstract}

\section{Introduction}

Consider the following abstract problem: given access to a
function $f : \mc{X} \to \R$, where $\mc{X}$ is some space, find
$x \in \mc{X}$ minimizing $f(x)$. We study an instantiation of this problem
that trades sequential access to $f$ for large batches of parallel
queries---one can query $f$ for its value over $n$ points at each of $T$
rounds.  In this setting, we propose a general algorithm that effectively
optimizes $f$ whenever there is a family of classifiers $h : \mc{X} \to
[0,1]$ that can predict sublevel sets of $f$ with high enough accuracy.

Our main motivation comes from settings in which $n$ is large---on the order
of hundreds to thousands---while possibly small relative to the size of
$\mc{X}$. These types of problems occur in biological
assays~\cite{knight2008array}, physical
simulations~\cite{marsden2004optimal}, and reinforcement learning
problems~\cite{schulman2015trust} where parallel computation or
high-throughput measurement systems allow efficient collection of large
batches of data.  More concretely, consider the optimization of protein
binding affinity to DNA sequence targets from biosensor
data~\cite{chevalier2017massively,knight2008array,wang2014particle}. In this
case, assays measure binding of $n \ge 1000$s of sequences and are
inherently parallel due to the fixed costs of setting up an experiment,
while the time to measure a collection of sequences makes multiple
sequential tests prohibitively time-consuming (so $T$ must be small).
In such problems, it is typically difficult to compute the gradients of $f$
(if they even exist); consequently, we focus on derivative-free optimization
(DFO, also known as zero-order optimization) techniques.

\subsection{Problem statement and approach}
The batched derivative free optimization problem consists of a sequence of rounds $t = 1, 2, \ldots, T$ in which we propose a distribution $p\sups{t}$, draw a sample of $n$ candidates $X_i \simiid p\sups{t}$, and observe $Y_i = f(X_i)$.
The goal is to find at least one example $X_i$ for which the gap
\begin{equation*}
  \min_i f(X_i) - \inf_{x \in \mc{X}} f(x)
\end{equation*}
is small.

Our basic idea is conceptually simple: In each round, fit a classifier $h$
predicting whether $Y_i \lessgtr \alpha\sups{t}$ for some threshold
$\alpha\sups{t}$.  Then, upweight points $x$ that $h$ predicts as $f(x) <
\alpha\sups{t}$ and downweight the other points $x$ for the proposal
distribution $p\sups{t}$ for the next round.

This algorithm is inspired by classical cutting-plane
algorithms~\cite[Sec.~3.2]{Nesterov04}, which remove a constant fraction of
the remaining feasible space at each iteration, and is extended into the
stochastic setting based on multiplicative weights
algorithms~\cite{Littlestone91,AroraHaKa12}. We present
the overall algorithm as Algorithm~\ref{alg:cutplane1}.

\begin{algorithm}[ht]
  \caption{Cutting-planes using classifiers}
  \label{alg:cutplane1}
  \begin{algorithmic}[1]
    \REQUIRE Objective $f$, Action space $\mathcal{X}$, hypothesis class $\mathcal{H}$.
  \STATE Set $p^{(0)}(x) = 1/|\mathcal{X}|$
  \STATE Draw $X^{(0)} \sim p^{(0)}$.
  \STATE Observe $Y^{(0)} = f(X^{(0)})$
  \FOR{$t\in\{1\hdots T\}$}
  \STATE Set $\alpha^{(t)} = \text{median}(\{Y^{(t)}_i\}_{i=1}^n)$
  \STATE Set $h^{(t)}\in\mathcal{H}$ as the loss minimizer of $L$ over
  $(X^{(0)},Y^{(0)}>\alpha^{(t)}) \hdots (X^{(t-1)},Y^{(t-1)}>\alpha^{(t)})$.
  \STATE Set $p^{(t)}(x) \propto p^{(t-1)}(x) (1-\eta h^{(t)}(x))$
  \STATE Draw $X^{(t)} \sim p^{(t)}$
  \STATE Observe $Y^{(t)} = f(X^{(t)})$.
  \ENDFOR
  \STATE Set $i^* = \arg\min_i Y_i^{(T)}$
  \RETURN $X_{i^*}^{(T)}$.
\end{algorithmic}
\end{algorithm}

\subsection{Related work}

When, as is typical in optimization, one has substantial \emph{sequential}
access to $f$, meaning that $T$ can be large, there are a number of major
approaches to optimization. Bayesian
optimization~\cite{shahriari2016taking,bogunovic2016truncated} and
kernel-based bandits~\cite{bubeck2016multi} construct an
explicit surrogate function to minimize; often, one assumes it is
possible to perfectly model the function $f$. Local search
algorithms~\cite{ConnScVi09,loshchilov2013cma} emulate gradient descent via
finite-difference and local function evaluations. Our work differs
conceptually in two ways: first, we think of $T$ as being small, while $n$
is large, and second, we represent a function $f$ by approximating its sublevel
sets. Existing batched derivative-free optimizers encounter computational
difficulties for batch sizes beyond dozens of
points~\cite{gonzalez2016batch}. Our sublevel set approach scales to large
batches of queries by simply sampling from the current sublevel set
approximation.

While other researchers have considered level set estimation in the context
of Bayesian optimization~\cite{gotovos2013active,bogunovic2016truncated} and
evolutionary algorithms~\cite{michalski2000learnable}, these use the level set
to augment a traditional optimization algorithm. We show good sublevel set
predictions alone are sufficient to achieve linear convergence. Moreover, given
the extraordinary empirical success of modern classification algorithms, e.g.\
deep networks for image classification~\cite{LeCunBeHi15}, it is natural to
develop algorithms for derivative-free optimization based on fitting a sequence
of classifiers.  \citet{yu2016derivative} also propose classification based on
optimization, but their approach assumes a classifier constrained to never
misclassify near the optimum, making the problem trivial.

\subsection{Contributions}

We present Algorithm~\ref{alg:cutplane1} and characterize its convergence rate
with appropriate classifiers and show how it relates to measures of difficulty
in active learning.  We extend this basic approach, which may be computationally
challenging, to an approach based on bootstrap resampling that is empirically
quite effective and---in certain nice-enough scenarios---has provable guarantees
of convergence.

We provide empirical results on a number of different tasks: random (simulated)
problems, airfoil (device) design based on physical simulators, and finding
strongly-binding proteins based on DNA assays.  We show that a black-box
approach with random forests is highly effective within a few rounds $T$ of
sequential classification; this approach provides advantages in the large batch
setting.

The approach to optimization via classification has a number of practical
benefits, many of which we verify experimentally. It is possible to incorporate
prior knowledge in DFO through domain-specific classifiers, and in more generic
optimization problems one can use black-box classifiers
such as random forests. Any sufficiently accurate classifier guarantees
optimization performance and can leverage the large-batch data collection
biological and physical problems essentially necessitate. Finally, one does not
even need to evaluate $f$: it is possible to apply this framework with pairwise
comparison or ordinal measurements of $f$.

\section{Cutting planes via classification}

Our starting point is a collection of ``basic'' results that apply to
classification-based schemes and associated convergence results.  Throughout
this section, we assume we fit classifiers using pairs $(x, z)$, where $z$
is a $0/1$ label of negative (low $f(x)$) or positive (high $f(x)$) class.
We begin by demonstrating that two quantities govern the convergence of the
optimizer: (1) the frequency with which the classifier misclassifies (and
thus downweights) the optimum $x^*$ relative to the multiplicative weight
$\eta$, and (2) the fraction of the feasible space each iteration removes.

If the classifier $h^{(t)}(x)$ exactly recovers the sublevel set
($h^{(t)}(x) < 0$ iff $f(x) < \alpha^{(t)}$), $\alpha^{(t)}$ is at most the
population median of $f(X^{(t)})$, and $\mc{X}$ is finite, the basic
cutting plane bound immediately implies that
\begin{multline*}
  \log \left[\P_{x\sim p^{(T)}}\left(
    f(x) = \min_{x^*\in\mathcal{X}} f(x^*)\right)
    \right]\\
  \geq \min\left( T \log \left(\frac{2}{2 -\eta}\right)
    - \log(|\mathcal{X}|), 0\right).
\end{multline*}
It is not obvious that such a guarantee continues to hold for inaccurate
$h^{(t)}$: it may accidentally misclassify the optimum $x^*$, and the
thresholds $\alpha^{(t)}$ may not rapidly decrease the function value. To
address these issues, we provide a careful analysis in the coming sections:
first, we show the convergence guarantees implied by Algorithm~\ref{alg:cutplane1}
as a function of classification errors (Theorem \ref{thm:comp-infeasible}), after
which we propose a classification strategy directly controlling errors
(Sec.~\ref{sec:css}), and finally we give a computationally tractable
approximation (Sec.~\ref{sec:bootstrap}).

\subsection{Cutting plane style bound}

We begin with our basic convergence result. 
Letting $p\sups{t}$ and $h\sups{t}$ be a sequence of distributions
and classifiers on $\mc{X}$, the convergence
rate depends on two quantities: the coverage (number of items cut)
\begin{equation*}
  \sum_{x \in \mc{X}} h\sups{t}(x) p\sups{t-1}(x)
\end{equation*}
and the number of times a hypothesis downweights item $x$ (because $f(x)$ is
too large), which we denote $M_T(x) \defeq \sum_{t = 1}^T h\sups{t}(x)$.
We have the following

\begin{restatable}{thm}{thmbasiccuttingplane}
\label{thm:comp-infeasible}
  Let $\gamma > 0$ and assume
  that for all $t$,
  \begin{equation*}
    \sum_{x \in \mathcal{X}}h^{(t)}(x) p^{(t-1)}(x) \geq \gamma
  \end{equation*}
  where $p\sups{t}(x) \propto p^{(t-1)}(x) (1-\eta h\sups{t}(x))$ as in
  Alg.~\ref{alg:cutplane1}.  Let $\eta \in [0,1/2]$ and $p^{(0)}$ be
  uniform. Then for all $x \in \mc{X}$,
  \begin{equation*}
    \log p\sups{T}(x)
    \ge \frac{\gamma \eta}{\eta + 2} T
    - \eta(\eta + 1) M_T(x) - \log(2 |\mc{X}|).
  \end{equation*}
\end{restatable}
The theorem follows from a modification of standard multiplicative weight
algorithm guarantees~\cite{AroraHaKa12}; see supplemental section
\ref{sec:cuttingplanes} for a full proof.

We say that our algorithm converges \emph{linearly} if $\log p\sups{t}(x)
\gtrsim t$. In the context of Theorem~\ref{thm:comp-infeasible},
choice of $\eta$ maximizing
$-(\eta^2+\eta)M_T(x^*)+ \frac{\eta}{\eta+2}\gamma T$
yields such convergence, as picking $\eta$
sufficiently small that
\begin{equation*}
  T - \frac{(\eta+1)(\eta+2)}{\gamma}M_T(x^*) = \Omega(T)
\end{equation*}
guarantees linear convergence if $2 M_T(x^*) < T \gamma$.
 
A simpler form of the above bound for a fixed $\eta$ shows the linear
convergence behavior.
\begin{restatable}{cor}{convergencecorollary}
  Let $x \in \mathcal{X}$, where $q_T(x) \defeq \frac{M_T(x)}{\gamma T}
  \leq 1 /4$. Under the conditions of Theorem
  \ref{thm:comp-infeasible},
  \[
    \log(p^{(T)}(x)) \geq \min\left(\frac{1}{5},
      \frac{1}{3}-\frac{4q_T(x)}{3}\right)
    \frac{\gamma T}{2} -\log(2|\mathcal{X}|)
  \]
  and
  \begin{equation*}
    \frac{1}{4} - \frac{\log(2|\mathcal{X}|) }{2\gamma T} \leq q_T(x).
  \end{equation*}
\end{restatable}
\noindent
The condition $q_T(x) \geq \frac{1}{4} - \frac{1}{2\gamma T}
\log(2|\mathcal{X}|)$ arises because if $M_T(x)$ is small, then eventually
we must have $p\sups{T}(x) \ge 1-\gamma$, and any classifier $h$ which
fulfils the condition $\sum_{x \in \mathcal{X}}h^{(t)}(x) p^{(t-1)}(x) \geq
\gamma$ in Thm.~\ref{thm:comp-infeasible} must downweight $x$.  At this
point, we can identify the optimum exactly with $O(1/(1-\gamma))$ additional
draws.


The corollary shows that if $M_T(x^*)=0$ and $\gamma = (1-1/e)-1/2
< 0$, we recover a linear cutting-plane-like convergence
rate~\cite[cf.][]{Nesterov04}, which makes constant progress in volume
reduction in each iteration.

\subsection{Consistent selective strategy for strong control of error}
\label{sec:css}

The basic guarantee of Theorem~\ref{thm:comp-infeasible} requires 
relatively few mistakes on $x^*$, or at least on a point $x$ with $f(x)
\approx f(x^*)$, to achieve good performance in optimization.  It is thus
important to develop careful classification strategies that are
conservative: they do not prematurely cut out values $x$ whose performance is
uncertain. With this in mind, we now show how consistent selective
classification strategies~\cite{el2012active} (related to active learning
techniques, and which abstain on ``uncertain'' examples similar to the
Knows-What-It-Knows framework~\cite{LiLiWa09,AbernethyAmDrKe13}) allow us to
achieve linear convergence when the classification problems are realizable using
a low-complexity hypothesis class.

The central idea is to only classify an example if all zero-error hypotheses
agree on the label, and otherwise abstain. Since any hypothesis achieving
zero population error must have zero training set errors, we will only
label points in a way consistent with the true labels.
\citet{el2012active} define the following
\emph{consistent selective strategy} (CSS).
\begin{defn}[Consistent selective strategy]
  \label{defn:css}
  For a hypothesis class $\mathcal{H}$ and training sample $S$, the
  \emph{version space} $\verspace_{\mathcal{H},S_m}\subset \mathcal{H}$ is
  the set of all hypotheses which perfectly classify $S_m$.  The
  \emph{consistent selective strategy} is the classifier
  \[h(x)=
  \begin{cases}
    1 &\text{ if }\forall g \in \verspace_{\mathcal{H},S_m}, g(x)=1 \\
    0 &\text{ if }\forall g \in \verspace_{\mathcal{H},S_m}, g(x)=0 \\
    \text{no decision} & \text{ otherwise.}
    \end{cases}
  \]
\end{defn}

Applied to our optimizer, this strategy enables safely downweighting
examples whenever they are classified as being outside the sublevel
set. Optimization performance guarantees then come from demonstrating that
at each iteration the selective strategy does not abstain on too many
examples.

The rate of abstention for a selective classifier is related to the
difficulty of disagreement based active learning, controlled by the
disagreement coefficient \cite{hanneke2014theory}.

\begin{defn}
  \label{def:disagree}
  The \emph{disagreement ball} of a hypothesis class $\mathcal{H}$ for
  distribution $P$ is
  \begin{equation*}
    B_{\mathcal{H},P}(h,r) \defeq \{h' \in \mathcal{H} \mid
    P(h(X)\neq h'(X)) \leq r\}.
  \end{equation*}
  The \emph{disagreement region of a subset $\mathcal{G}\subset \mathcal{H}$} is
  \begin{equation*}
    \disagree(\mathcal{G}) \defeq
    \{x\in\mathcal{X} \mid \exists h_1, h_2 \in
    \mathcal{G} \text{ s.t. } h_1(x) \neq h_2(x)\}.
  \end{equation*}
  The \emph{disagreement coefficient} $\discoeff_h$
  of the hypothesis class $\mathcal{H}$
  for the distribution $P$ is
  \begin{equation*}
    \discoeff_h \defeq \sup_{r > 0}
    \frac{P(X\in\disagree(B_{\mc{H},P}(h,r)))}{r}.
  \end{equation*}
\end{defn}
The disagreement coefficient directly bounds the abstention rate as a function of generalization error.
\begin{restatable}{thm}{thmcsscover}
  \label{thm:csscover}
  Let $h$ be the CSS classifier in definition \ref{defn:css}, and let $h^*\in\mathcal{H}$ be a classifier achieving zero risk.  If
  $\P(g(X) \neq h^*(X)) < \epsilon$ for all $g
  \in \verspace_{\mathcal{H},S_m}$, then CSS achieves coverage
  \begin{equation*}
    \P(h(X) = \text{no decision}) \leq \discoeff_{h^*} \epsilon
  \end{equation*}
\end{restatable}
This follows from the definition of the disagreement coefficient, and the
size of the version space (Supp. section \ref{sec:cuttingplanes} contains a
full proof).

The dependence of our results on the disagreement coefficient implies a
reduction from zeroth order optimization to disagreement based active
learning~\cite{el2012active} and selective
classification~\cite{wiener2011agnostic} over sublevel sets.

Implementing the CSS classifier may be somewhat challenging: given a
particular point $x$, one must verify that all hypotheses consistent with
the data classify it identically. In many cases, this requires training a
classifier on the current training sample $S\sups{t}$ at iteration $t$,
coupled with $x$ labeled positively, and then retraining the classifier with
$x$ labeled negatively~\cite{wiener2011agnostic}. This cost can be
prohibitive. (Of course, implementing the multiplicative weights-update
algorithm over $x \in \mc{X}$ is in general difficult as well, but in a
number of application scenarios we know enough about $\mc{H}$ to be able to
approximate sampling from $p\sups{t}$ in Alg.~\ref{alg:cutplane1}.)

A natural strategy is to use the CSS classifier as part of Algorithm
\ref{alg:cutplane1}, setting all \texttt{no decision} outputs to the zero
class, only removing points confidently above the level set
$\alpha\sups{t}$. That is, in round $t$ of the algorithm,
given samples $S=(X\sups{t}, Z\sups{t})$, we define
\begin{equation*}
  h^{(t)}(x) = \begin{cases}
    1 &\text{ if }\forall g \in \verspace_{\mathcal{H},S}, g(x)=1\\
    0 &\text{ if }\forall g \in \verspace_{\mathcal{H},S}, g(x)=0\\
    0 & \text{ otherwise.}
  \end{cases}
\end{equation*}
There is some tension between classifying examples correctly and cutting
out bad $x \in \mc{X}$, which the next theorem shows we can address
by choosing large enough sample sizes $n$.
\begin{restatable}{thm}{thmcsscut}
  \label{thm:csscut}
  Let $\mathcal{H}$ be a hypothesis class containing indicator functions for
  the sublevel sets of $f$, with VC-dimension $V$ and disagreement
  coefficient $\discoeff_h$.  There exists a numerical constant $C < \infty$
  such that for all $\delta \in [0, 1]$, $\epsilon \in [0, 1]$,
  and $\gamma \in (\discoeff_h \epsilon, \half)$, and
  \begin{multline*}
    n \geq \max\Big\{C\epsilon^{-1} [V \log(\epsilon^{-1}) + \log(\delta^{-1}) + \log(2T)],\\
    \frac{1}{2(\gamma-0.5)^2}(\log(\delta^{-1})+\log(2T))\Big\},
  \end{multline*}
  with probability at least $1 - \delta$
  \begin{multline*}
    \log(p^{(T)}(x^*)) \geq \min\Big\{(\gamma-\discoeff_h\epsilon) \frac{\eta}{\eta+2}T -\log(2|\mathcal{X}|), \\
    \log(1-\gamma) \Big\}
  \end{multline*}
  after $T$ rounds of Algorithm~\ref{alg:cutplane1}.
\end{restatable}
The proof follows from combining the selective classification bound with
standard VC dimension arguments to obtain the sample size requirement
(Supp.~\ref{sec:cuttingplanes} contains a full proof).


Thus if $\discoeff_h$ is small, such as $\log(|\mathcal{X}|)$, then choosing
$\epsilon = \discoeff_h^{-1}$ achieves exponential improvements over random
sampling. In the worst case, $\discoeff_h=O(|\mathcal{X}|)$, but small $\discoeff_h$ are known for many problems, for example for linear classification with continuous $\mathcal{X}$ over densities bounded away from zero, $\discoeff_h =
\text{poly}(\log(\text{Vol}(\mathcal{X})))$, which would result in linear
convergence rates (Theorem 7.16, \cite{hanneke2014theory}).

Using recent bounds for the disagreement coefficient for linear separators
\cite{BalcanLo13}, we can show that for linear optimization over a convex
domain, the CSS based optimization algorithm above achieves linear
convergence with $O(d^{3/2}\log(d^{1/2})-d^{1/2}\log(3T\delta))$ samples
with probability at least $1-\delta$ (for lack of space, we present this as
Theorem~\ref{thm:linear-opt-classifier} in
the supplement.)

When the classification problem is non-realizable, but the Bayes-optimal
hypothesis does not misclassify $x^*$, an analogous result holds
through the agnostic selective classification framework of Wiener and El-Yaniv
\cite{wiener2011agnostic}. The
full result is in supplemental Theorem~\ref{thm:cssagn}.

\section{Computationally efficient approximations}
\label{sec:bootstrap}

While selective classification provides sufficient
control of error for linear convergence, it is generally computationally
intractable. However, a bootstrap resampling
algorithm~\cite{EfronTi93} approximates selective classification well enough to
provide finite sample guarantees in parametric settings.
Our analysis provides intuition for the empirical observation that selective classification via the bootstrap works well in many real-world problems~\cite{mamitsuka1998query}.

Formally, consider a parametric family $\{P_\theta\}_{\theta \in
  \Theta}$ of conditional distributions $Z \mid X \in [0,1]$ with compact parameter space $\Theta$. Given $n$ samples $X_1,\dots,X_n$, we observe $Z_i |
X_i \sim P_{\theta^*}$ with $\theta^* \in
\operatorname{int} \Theta$.

Let $\ell_\theta(x, z) = -\log(P_\theta(z | x))$ be the negative
log likelihood of $z$, which majorizes the 0-1
loss of the linear hypothesis class $\ell_\theta(x, z) \ge \ind{ (2z-1) x^\top
\theta < 0 }$. 

Define the weighted likelihood
\[
  L_n(\theta,u) \equiv \tfrac{1}{n} \sum_{i=1}^n (1+u_i)\ell_\theta(X_i, Z_i),
\]
and consider the following multiplier bootstrap
algorithm~\cite{EfronTi93,spokoiny2012parametric}, parameterized by $B \in
\naturals$ and variance $\sigma^2$. $\sigma$ adds \emph{additional}
variation in the estimates to increase parameter coverage.
\begin{enumerate}
\item Draw $\{(X_i, Z_i)\}_{i=1}^n$ from $\P$.
  \item Compute $\mle = \arg\min_\theta L_n(\theta, 0)$.
  \item For $b = 1$ to $B$,
    \begin{enumerate}
      \item Draw $u_b \simiid \mbox{Uni}[-1, 1]$.
      \item Compute
        \[\theta^\circ_{u_b} = \sigma(\arg\min_\theta L_n(\theta, u_b)-\mle)+\mle.\]
    \end{enumerate}
  \item Define the estimator \[ h^\circ(x) =  \begin{cases}
      1 &\text{ if }\forall b\in[B], x^\top \theta^\circ_{u_b}> 0 \\
      0 &\text{ if }\forall b\in[B], x^\top \theta^\circ_{u_b} \leq 0\\
         \text{no decision} & \text{ otherwise.}
    \end{cases}.\]
\end{enumerate}

For linear classifiers with strongly convex losses,
this algorithm obtains selective classification guarantees under appropriate
regularity conditions as presented in the following theorem.

  \begin{thm}
    Assume $\ell_\theta$ is twice differentiable and fulfils $\|\nabla \ell_{\theta}(X,
Z)\| \le R$, and $\norm{\nabla^2\ell_\theta(X,Z)}_{op} \leq S$ almost surely.
      Additionally, assume $L_n(\theta, 1)$ is $\gamma$-strongly convex and
      that $\nabla^2 L_n(\theta, 1)$ is $M$-Lipschitz with probability one.

      For $h^\circ$ defined above and $x\in\mathcal{X}$,
      \[P(x^\top \theta^* \leq 0 \text{ and } h^\circ_{u}(x) = 1) < \delta.\]
      Further, the abstention rate is bounded by
      \[\int_{x\in \reals^d}\ind{h^\circ_u(x)=\emptyset}p(x)dx \leq \epsilon \discoeff_h\]
      with probability $1-\delta$ whenever
      \[B \geq 15\log(3/\delta),\]
      \[\sigma=O(d^{1/2} + \log(1/\delta)^{1/2}+n^{-1/2}),\]
      \[\epsilon=O\left(\sigma^2 n^{-1} \log(B/\delta)\right),\]
      and
      \[n  \geq 2\log(2d/\delta)S/\gamma^2.\]
\end{thm}
Due to length, the proof and full statement with constants appears in the appendix as Theorem \ref{thm:bootlin},
with a sketch provided here: we first show that a given quadratic version space and
a multivariate Gaussian sample $\tq$ obtains the selective classification guarantees (Lemmas \ref{lem:quadmin},\ref{lem:ballmax},\ref{lem:strong-convex}).
We then show that $\theta^\circ \approx \tq$ to order $n^{-1}$ which is sufficient to
recover Theorem \ref{thm:bootlin}.

\begin{figure}[h]
  \centering
\subcaptionbox{Classification confidences formed by bootstrapping approximate
  selective classification.\label{fig:bootbound}}
{\includegraphics[scale=0.27]{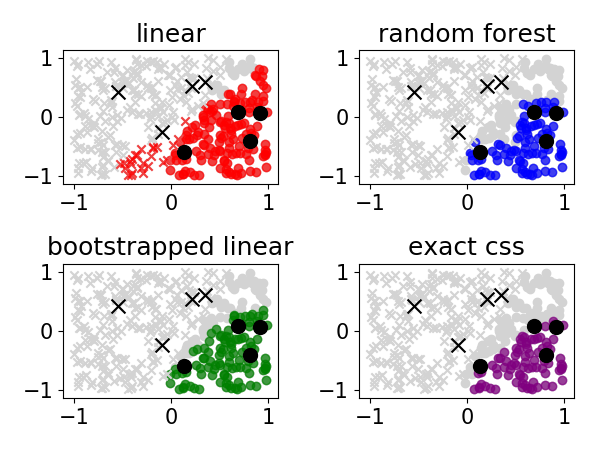}}
\quad
\subcaptionbox{Bootstrapping results in more consistent identification of
  minima. \label{fig:bootopt}}
{\includegraphics[scale=0.24]{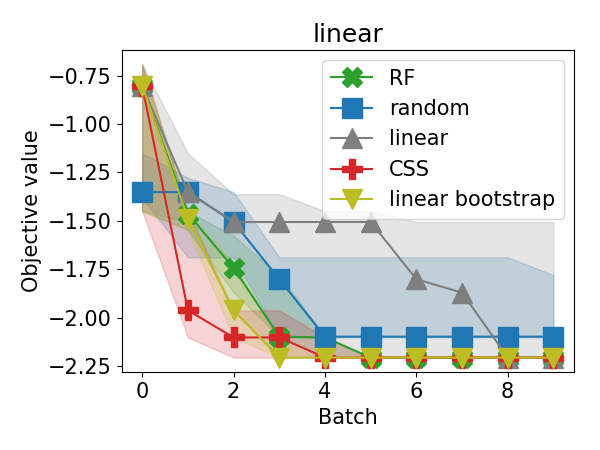}}
\caption{Bootstrap consensus provides more conservative classification boundaries which prevents repeatedly misclassifying the minimum, compared to direct loss minimization (panel b, triangle).}
\end{figure}

The $d \discoeff_h$ abstention rate in this bound is $d$
times the original selective classification result. This additional factor
of $d$ appearing in $\sigma^2$ arises from the difference between
finding an optimum within a ball and randomly sampling it: random vectors
concentrate within $O(1/d)$ of the origin, while the maximum possible value
is 1. This gap forces us to scale the variance in the decision function
by $\sigma$ (step 3b).
We present selective classification approximation bounds analogous
to Theorem~\ref{thm:csscut} for linear optimization
in the Appendix as Theorem~\ref{thm:csscutboot}.

To illustrate our results through simulations, consider a optimizing a
two-dimensional linear function in the unit box. Figure \ref{fig:bootbound}
shows the set of downweighted points (colored points) for various algorithms
on classifying a single superlevel set based on eight observations (black
points). Observe how linear downweights many points (colored `x'), in
contrast to exact CSS, which only downweights points guaranteed to be in the
superlevel set. Errors of this type combined with
Alg.~\ref{alg:cutplane1} result in optimizers which fail to find the
true minimum depending on initialization (Figure \ref{fig:bootopt}). The
bootstrapped linear classifier behaves similarly to CSS, but is looser due
to the non-asymptotic setting. Random forests, another type of bootstrapped
classifier is surprisingly good at approximating CSS, despite not making use
of the linearity of the decision boundary.

\section{Partial order based optimization}
One benefit of optimizing via classification is that the algorithm
only requires total ordering amongst the elements. Specifically, step 6 of Algorithm
\ref{alg:cutplane1} only requires threshold comparisons against a percentile selected in
step 5. This enables optimization under pairwise comparison feedback. At each round, instead of observing $f(X^{(t)})$, we observe $g(X_i^{(t)}, X_j^{(t)})=1_{f(X_i^{(t)}) < f(X_j^{(t)})}$, which is a natural form of feedback in domains such as human surveys \cite{phelps2015pairwise} or matched biological experiments \cite{harwood2013microbial}.

Given the pairwise comparison function $g$, the threshold $f(X^{(t)}) <
\alpha^{(t)}$ can be replaced with the following stochastic quantile estimator:
\begin{equation}\label{eq:paircomp}
  \hat{f}(X^{(t)}_i) = \sum_{k=1}^c g(X^{(t)}_{I_k}, X^{(t)}_i) \leq 0.5,
  \end{equation}
  where $I_k \sim \text{Unif}(\{1, 2 \hdots c\})$ with $cn$ total pairwise comparisons. We show that $c > 10$ seems to work well in practice, and more sophisticated preference aggregation algorithms may reduce the number of comparisons even further.
  
\section{Experimental evidence}
We evaluate Algorithm \ref{alg:cutplane1} as a DFO algorithm across
a few real-world experimental design benchmarks,
common synthetic toy optimization problems, and benchmarks that allow only
pairwise function value comparisons. The small-batch (n = 1-10) nature of
hyperparameter optimization problems is outside the scope of our work, even
though they are common DFO problems.

For constructing the classifier in Algorithm \ref{alg:cutplane1}, we apply
ensembled decision trees with a consensus decision defined as 75\% of trees
agreeing on the label (referred to as \textsc{classify-rf}). This particular
classifier works in a black-box setting, and is highly effective across all
problem domains with no tuning. We also empirically investigate the
importance of well-specified hypotheses and consensus ensembling and show
improved results for ensembles of linear classifiers and problem specific
classifiers, which we call \textsc{classify-tuned}.

In order to demonstrate that no special tuning is necessary, the same constants
are used in the optimizer for all experiments, and the classifiers use
off-the-shelf implementations from \textsc{scikit-learn} with no tuning.

For sampling points according to the weighted distribution in Algorithm
\ref{alg:cutplane1}, we enumerate for discrete action spaces $\mathcal{X}$, and for
continuous $\mathcal{X}$ we perturb samples from the previous rounds using a
Gaussian and use importance sampling to approximate the target
distribution. Although exact sampling for the continuous case would be
time-consuming, the Gaussian perturbation heuristic is fast, and seems to work
well enough for the functions tested here.

\begin{figure*}[h]
  \centering
\subcaptionbox{Binding to the CRX protein\label{fig:pbm1}}
{\includegraphics[scale=0.35]{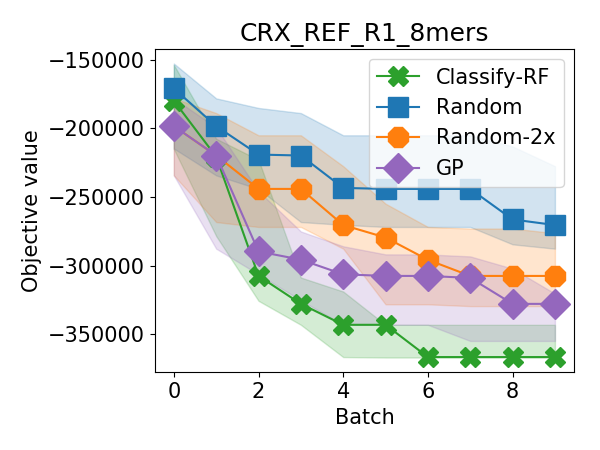}}
\subcaptionbox{Binding to the VSX1 protein\label{fig:pbm2}}
{\includegraphics[scale=0.35]{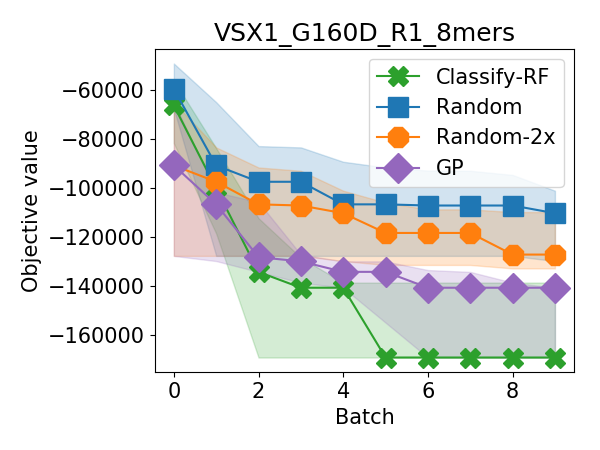}}
\subcaptionbox{High-lift airfoil design\label{fig:airfoil}}
{\includegraphics[scale=0.35]{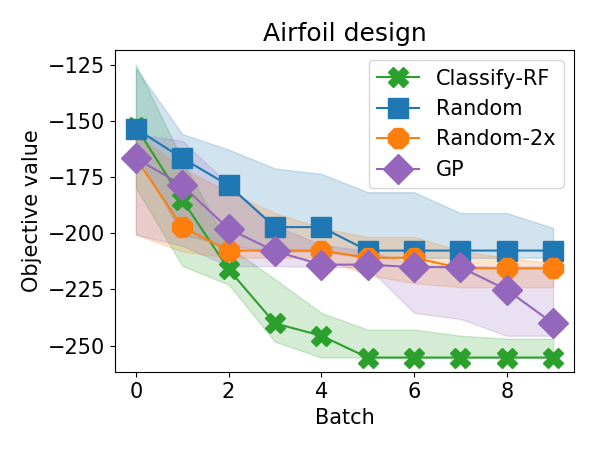}}
\caption{Performance on two types of real-world batched zeroth-order
  optimization tasks. \textsc{classify-rf} consistently outperforms baselines
  and even randomly sampling twice the batch size. The line shows median
  function value over runs, shaded area is quartiles.}
\label{fig:pbm}
\end{figure*}

As a baseline, we compare to the following algorithms\vspace{-2ex}
\begin{itemize}
\item Random sampling (\textsc{random})
  \vspace{-1ex}
\item Randomly sampling double the batch size (\textsc{random-2x}), which is a strong baseline
  recently shown to outperform many derivative-free optimizers \cite{li2016hyperband}.
  \vspace{-1ex}
\item The evolutionary strategy (\textsc{CMA-ES}) for continuous problems, due
  to its high-performance in black box optimization competitions as well as
  inherent applicability to the large batch setting \cite{loshchilov2013cma}
  \vspace{-3ex}
\item The Bayesian optimization algorithm provided by \textsc{GpyOpt}\cite{gpyopt2016} (\textsc{GP})
  for both continuous and discrete problems, using expected improvement as the
  acquisition function. We use the `random' evaluator which implements an
  epsilon-greedy batching strategy, since the large batch sizes (100-1000) makes
  the use of more sophisticated evaluators completely intractable. The default
  RBF kernel was used in all experiments presented here. The $\sfrac{3}{2}$- and
  $\sfrac{5}{2}$-Matern kernels and string kernels were tried where appropriate,
  but did not provide any performance improvements.
\vspace{-1ex}
\end{itemize}
In terms of runtime, all computations for \textsc{classify-rf} take less than 1
second per iteration compared to 0.1s for \textsc{CMA-ES} and 1.5 minutes for
\textsc{GpyOpt}. All experiments were replicated fifteen times to measure
variability with respect to initialization.

All new benchmark functions and reference implementations are made available at \url{http://bit.ly/2FgiIxA}.

\subsection{Designing optimal DNA sequences}


The publicly available protein binding microarray (PBM) dataset consisting of
201 separate assays \cite{barrera2016survey} allows us to accurately benchmark
the optimization protein binding over DNA sequences. In each assay, the binding
affinity between a particular DNA-binding protein (transcription factor) and all
8-base DNA sequences are measured using a microarray.

This dataset defines 201 separate discrete optimization problems. For each
protein, the objective function is the negative binding affinity (as measured by
fluorescence), the batch size is 100 (corresponding roughly to the size of a
typical 96-well plate), across ten rounds. Each possible action corresponds to
measuring the binding affinity of a particular 8-base DNA sequence exactly. The
actions are featurized by considering the binary encoding of whether a base
exists in a position, resulting in a 32-dimensional space. This emulates the
task of finding the DNA binding sequence of a protein using purely
low-throughput methods.

Figure \ref{fig:pbm1},\ref{fig:pbm2} shows the optimization traces of two randomly sampled
examples, where the lines indicate median achieved function value over 15 random initializations, and the shading indicates quartiles. \textsc{classify-rf} shows consistent improvements over all discrete
action space baselines. For evaluation, we further sample 20 problems and find
that the median binding affinity found across replicates is strictly better on 16 out of
20, and tied with the Gaussian process on 2.

In this case, the high performance of random forests is relatively unsurprising,
as random forests are known to be high-performance classifiers for DNA sequence
recognition tasks \cite{chen2012random,knight2008array}.

\begin{figure*}[h]
  \centering
\subcaptionbox{Random linear function\label{fig:lin}}
{\includegraphics[scale=0.35]{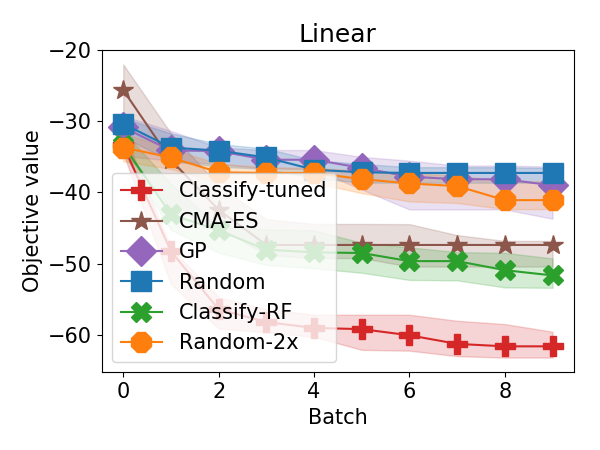}}
\subcaptionbox{Linear$+$quadratic function\label{fig:quad}}
{\includegraphics[scale=0.35]{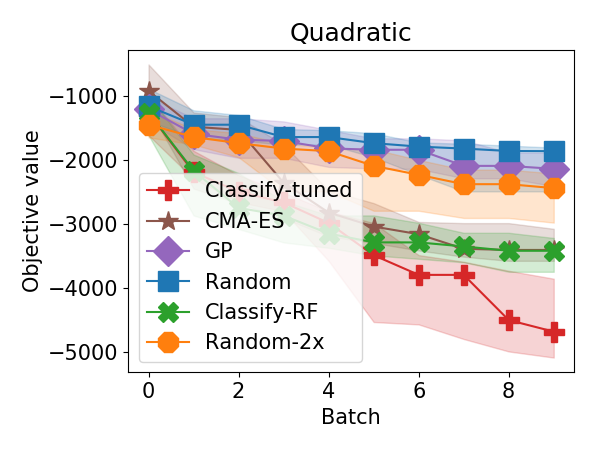}}
\subcaptionbox{Ensembling classifiers improves optimization performance\label{fig:ens}}
{\includegraphics[scale=0.35]{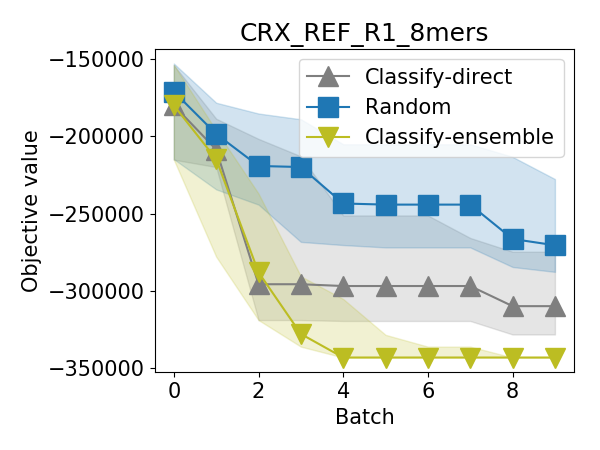}}
\caption{Testing the importance of ensembling and well-specified hypothesis class in synthetic data where the hypothesis for \textsc{Classify-tuned} exactly matches level set (panel a), matches level sets with some error (panel b). Ensembling also consistently improves performance, and reduces dependence on initialization (panel c)}
\end{figure*}

\subsection{Designing high-lift airfoils}

Airfoil design, and other simulator-based objectives are well-suited to the
batched, classification based optimization framework, as 30-40 simulations can
be run in parallel on modern multicore computers.
In the airfoil design case, the simulator is a 2-D aerodynamics simulator for
airfoils \cite{drela1989xfoil}.

The objective function is the negative of lift divided by drag (with a zero
whenever the simulator throws an error) and the action space is the set of all
common airfoils (NACA-series 4 airfoils). The airfoils are featurized by taking
the coordinates around the perimeter of the airfoil as defined in the Selig airfoil format. This results in a highly-correlated two hundred dimensional feature
space. The batch size is 30 (corresponding to the number of cores in our
machine) and $T=10$ rounds of evaluations are performed.

We find in Figure \ref{fig:airfoil} that the \textsc{classify-rf} algorithm converges to the
optimal airfoil in only five rounds, and does so consistently, unlike the baselines.
The Gaussian process beat the twice-random baseline, since the
radial basis kernel is well-suited for this task (as lift is relatively smooth
over $\ell_2$ distance between airfoils) but did not perform as well as the \textsc{classify-rf} algorithm.

\subsection{Gains from designed classifiers and ensembles}

Matching the classifier and objective function generally results in large
improvements in optimization performance. We test two continuous optimization
problems in $[-1,1]^{300}$, optimizing a random linear function, and optimizing
a random sum of a quadratic and linear functions. For this high dimensional
task, we use a batch size of 1000. In both cases we compare continuous baselines
with \textsc{classify-rf} and \textsc{classify-tune} which uses a linear
classifier.

We find that the use of the correct hypothesis class gives dramatic improvements
over baseline in the linear case (Figure \ref{fig:lin}) and continues to give
substantial improvements even when a large quadratic term is added, making the
hypothesis class misspecified (Figure \ref{fig:quad}). The \textsc{classify-rf}
does not do as well as this custom classifier, but continues to do as well as
the best baseline algorithm (\textsc{CMA-ES}).

We also find that using an ensembled classifier is an important for optimization.
Figure \ref{fig:ens} shows an example run on the DNA binding task  comparing the consensus
of an ensemble of logistic regression classifiers against a single logistic
regression classifier. Although both algorithms perform well in early iterations,
the single logistic regression algorithm gets `stuck' earlier and finds a
suboptimal local minima, due to an accumulation of 
errors. Ensembling consistently reduces such behavior.

\subsection{Low-dimensional synthetic benchmarks}

We additionally evaluate on two common synthetic benchmarks (Figure
\ref{fig:shekel},\ref{fig:hartmann}). Although these tasks are not the focus of
the work, we show that the \textsc{classify-rf} is surprisingly good as a
general black box optimizer when the batch sizes are large.

We consider a batch size of 500 and ten steps due to the moderate dimensionality
and multi-modality relative to the number of steps. We find qualitatively
similar results to before, with \textsc{classify-rf} outperforming other algorithms
and \textsc{CMA-ES} as the best baseline.

\begin{figure}[h]
  \centering
\subcaptionbox{Shekel (4d)\label{fig:shekel}}
{\includegraphics[scale=0.26]{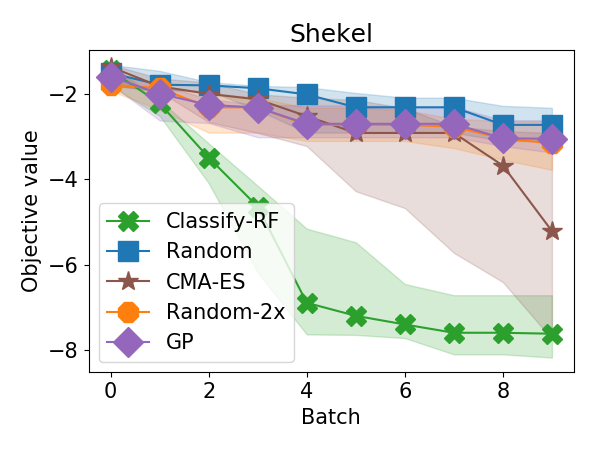}}
\subcaptionbox{Hartmann (6d)\label{fig:hartmann}}
{\includegraphics[scale=0.26]{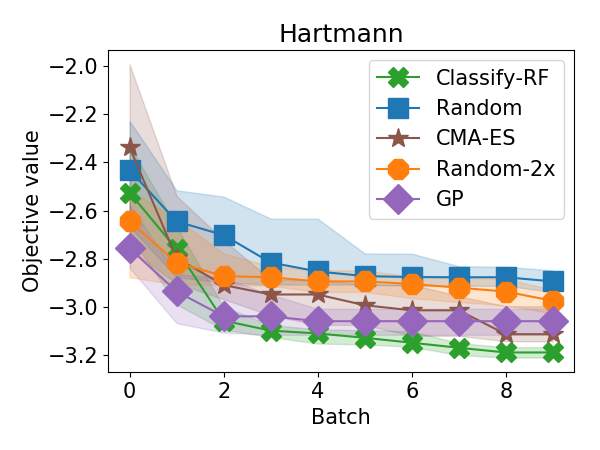}}
\caption{\textsc{classify-rf} outperforms baselines on synthetic benchmark
  functions with large batches}
\end{figure}

\subsection{Optimizing with pairwise comparisons}
Finally, we demonstrate that we can optimize a function using only
pairwise comparisons. In Figure \ref{fig:paircomp} we show the optimization performance when using the ordering estimator from equation \ref{eq:paircomp}.

For small numbers of comparisons per element $(c=5)$ we find substantial loss of
performance, but once we observe at least 10 pairwise comparisons per proposed
action, we are able to reliably optimize as well as the full function value
case. This suggests that classification based optimization can handle pairwise feedback with little loss in efficiency.

\begin{figure}[h]
  \centering
  \includegraphics[scale=0.3]{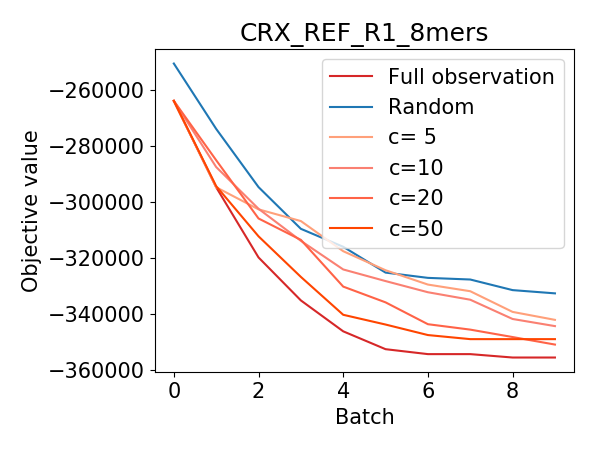}
  \caption{Optimization with pairwise comparisons between each action
    and a small set of $(c)$ randomly
    selected actions. Between 10-20 pairwise comparisons per action gives sufficient
    information to fully optimize the function.}
  \label{fig:paircomp}
\end{figure}

\section{Discussion}

Our work demonstrates that the classification-based approach to
derivative-free optimization is effective and principled, but leaves open
several theoretical and practical questions. In terms of theory, it is not
clear whether a modified algorithm can make use of empirical risk minimizers
instead of perfect selective classifiers. In practice, we have left the
question of tractably sampling from $p^{(t)}$, as well as how to
appropriately handle smaller-batch settings of $d > n$.

\clearpage


\bibliography{main,bib/bib.bib}{}
\bibliographystyle{abbrvnat}

\clearpage
\appendix

\section{Supplementary materials}

\subsection{Cutting plane algorithms}\label{sec:cuttingplanes}

\thmbasiccuttingplane*
\begin{proof}
 Since the sampling distribution is derived from multiplicative weights over
$\sum_{t=1}^T h^{(t)}(x)$, the following regret bound holds with respect to any
$p$ (Theorem 2.4, \cite{AroraHaKa12}):
 \begin{align*}
   \gamma T &\leq \sum_{t=1}^T \sum_{x\in \mathcal{X}} h^{(t)}(x) p^{(t-1)}(x) \\
            &\leq (1+\eta)\sum_{t=1}^T \sum_{x\in\mathcal{X}}h^{(t)}(x) p(x) +
\frac{\text{KL}(p||p^{(0)})}{\eta}
   \end{align*}
  Pick $S = \{i: \sum_{t}h^{(t)}(x) < \nu T\}$ and $p$ uniform over $S$ to get:
  \[\gamma T\leq (1+\eta)\nu T + \frac{\log(|\mathcal{X}|/|S|)}{\eta}\]
  From this we get the bound:
  \[\log(|S|) \leq  \log(|\mathcal{X}|) - \gamma \eta T + \nu \eta T  + \nu \eta^2 T\]

  Now we can get the following basic bound on $\log(p^{(T)}(x^*))$ by
  decomposing the normalizer using the set $S$.
  \begin{align*}
    \log(p^{(T)}(x^*)) \geq \log(1-\eta) M_T(x^*) \\
    - \log(|S| + \exp( \log(1-\eta)\nu  T)(|\mathcal{X}|-|S|))
    \end{align*}
  Note that
  \[\exp(\log(1-\eta)\nu T) < \exp(-\eta\nu T)\]
  as well as
  \[-\eta \nu T = -\gamma \eta T + \nu \eta T + \nu \eta^2 T\]
  whenever $\nu = \gamma/(\eta+2)$. This gives the normalizer bound:
  \begin{multline*}
    \log(|S| + \exp( \log(1-\eta)\nu  T)(|\mathcal{X}|-|S|)) \\
    < \log(2\exp(-\eta\gamma/(\eta+2) T)|\mathcal{X}|)
    \end{multline*}

  Combining the above:
  \begin{align*}
    &\log(p^{(T)}(x^*)) \\
    &\geq \log(1-\eta) M_T(x^*) - \log(|S| + \exp(\log(1-\eta) \nu T)|\mathcal{X}|)\\
      &\geq \log(1-\eta) M_T(x^*) - \log(2\exp(-\eta\gamma/(\eta+2)T)|\mathcal{X}|)\\
      &\geq \log(1-\eta) M_T(x^*) + \frac{\eta}{\eta+2}\gamma T -\log(2|\mathcal{X}|)\\
      &\geq -(\eta^2+\eta)M_T(x^*)+ \frac{\eta}{\eta+2}\gamma T -\log(2|\mathcal{X}|)\\
      &\geq \gamma \frac{\eta}{\eta+2}\left(T - \frac{(\eta+1)(\eta+2)}{\gamma}M_T(x^*)\right) -\log(2|\mathcal{X}|) 
  \end{align*}
  Where in the second step we use the normalizer bound, and in the third we use
  the identity $-\log(1-x) \leq x + x^2$ for $x \in [0,1/2]$.
  \end{proof}

\convergencecorollary*
  \begin{proof}
  First we balance the linear terms in Theorem \ref{thm:comp-infeasible} by
  solving for $\eta$ in
  \[\frac{(\eta+1)(\eta+2)}{\gamma} M_T(x) = \frac{T}{2}.\]
  This gives the multiplicative weight
  \[
    \eta = \min\left(\frac{1}{2},\sqrt{\frac{1}{4} + \frac{1}{2q_T(x)}}
      - \frac{3}{2}\right) > 0.
  \]
  The inequality follows from $q_T(x) \leq 1/4$ and this reduces the original
  bound to
  \begin{multline*}\log(p^{(T)}(x)) \geq \\
    \min\left(\frac{1}{5}, 1+2q_T(x)- 2\sqrt{2q_T(x) + q_T(x)^2}
    \right)\frac{\gamma T}{2}\\
    -\log(2|\mathcal{X}|).
    \end{multline*}
    Using the linear lower bound on $1+2q_T(x)- 2\sqrt{2q_T(x) + q_T(x)^2}$ at
    $q_T(x) = 1/4$ gives the first inequality.

    The second inequality follows from noting that $\log(p^{(T)}(x)) < 0$ and
    solving for $q_T(x)$.
  \end{proof}

  \thmcsscover*
\begin{proof}
  The inequality follows from two facts. The first fact is that $P(h(X) = \text{no decision})$ is
  upper bounded by the probability of sampling an $X$ is the disagreement ball of radius $\epsilon$.
  This follows from the definition of the CSS classifier which outputs no decision if and only if there exists a  classifier in $\verspace_{\mathcal{H},S_m}$ which does not agree with the
  others, which is the definition of the disagreement region.

  The second fact is that if
  $\sup_{g\in\verspace_{\mathcal{H},S_m}}P(g(x) \neq h^*(x))
  < \epsilon$, then $\verspace_{\mathcal{H},S_m} \subseteq
  B_{\mathcal{H},P}(h^*,\epsilon)$, by construction of the version space. Applying the definition of the
  disagreement coefficient completes the proof.
\end{proof}

\thmcsscut*
\begin{proof}
  The proof has three parts: first, we prove that the median of $n$ samples $Y_1
  \hdots Y_n$ is at least the $\gamma$ quantile over $p^{(t)}$.

  The probability that the $n$-sample empirical median is less than the $\gamma$
  quantile over $p^{(t)}$ is equivalent to having $X \sim \text{Binom}(\gamma, n)$
  and $\P(X > n/2)$. Applying Hoeffding's inequality,
  \[\P(X > n/2) \leq \exp(-2(\gamma-0.5)^2 n).\]

  Next, we show that the CSS prediction abstains on at most $\discoeff_h\epsilon$
  fraction of the distribution. VC dimension bounds
  \cite{kearns1994introduction} imply that we can achieve $\epsilon$ error
  uniformly over hypothesis class $\mathcal{H}$ with VC dimension $V$ in
  \[n = C\epsilon^{-1} [V \log(\epsilon^{-1}) + \log(\delta'^{-1})]\]
  samples with probability at least $1-\delta$ for some constant $C$. We can
  then apply the CSS classifier bound to get that the abstention rate is
  $\discoeff_h\epsilon$. This implies:
  \begin{equation}\label{eq:holdout}
    \sum_{x\in \mathcal{X}} h^{(t)}(x) p^{(t-1)}(x) \geq \gamma - \discoeff_h\epsilon
    \end{equation}
    
  Finally, whenever $p^{(t)}(x^*) < \frac{1}{1-\gamma}$, 
  \[h(x^*) = 1_{f(x^*) \leq \alpha^{(t)}} = 0,\]
  and thus $M_t(x^*) = 0$.

  The $\log(p^{(T)})$ inequality follows by applying Theorem
  \ref{thm:comp-infeasible} with Equation \ref{eq:holdout} noting that either
  $M_T(x^*)=0$ or $p^{(T)}(x^*) > \frac{1}{1-\gamma}$.

  Union bounding the above two probabilities, and ensuring each part has failure
  probability $\delta/2$,
  \begin{multline*}
    n = \max\big(C\epsilon^{-1} [V \log(\epsilon^{-1}) + \log(\delta^{-1}) + \log(2T)],\\
      \frac{1}{2(\gamma-0.5)^2}(\log(\delta^{-1})+\log(2T))\big)
    \end{multline*}
  implies $\delta$ failure probability.
\end{proof}

\subsection{Convergence rates for optimizing linear functions}.

The first lemma generalizes Theorem \ref{thm:comp-infeasible} to the continuous case. 
\begin{lem}
  Consider a compact $\mathcal{X} \subset \mathbb{R}^d$ and $\gamma > 0$, and
  let $(h^{(t)}: \mathcal{X} \to [0,1], p^{(t-1)})_{t \in \naturals}$ be a
  sequence such that for every $t \in \naturals$,
  \begin{gather*}
    \int_x h^{(t)}(x) p^{(t-1)}(x) dx \geq \gamma,\\
      p^{(t)}(x) \propto p^{(t-1)}(x) (1-\eta h^{(t)}(x))
  \end{gather*}
  for some $\eta \in [0,1/2]$, and $p^{(0)}$ the uniform distribution over
  $\mathcal{X}$. Further, let $M_T(x) = \sum_{t=1}^T h^{(t)}(x)$ the number of
  times $h^{(t)}$ downweights item $x$. Then the following bound on density at
  the any item $x^*$ holds on the last step:
  \begin{multline*}
    \log(p^{(T)}(x)) \geq \gamma \frac{\eta}{\eta+2}\left(T - \frac{(\eta+1)(\eta+2)}{\gamma}M_T(x)\right)\\
    -\log(\vol(\mathcal{X})/2)\\
  \end{multline*}
\label{lem:cont-cut}
\end{lem}
\begin{proof}
  Since the sampling distribution is derived from multiplicative weights over
  $\sum_{t=1}^T h^{(t)}(x)$, the following regret bound holds with respect to
  any $p$ (Theorem 2.4, \cite{AroraHaKa12}):
 \begin{align*}
   \gamma T &\leq \sum_{t=1}^T \int_{x\in \mathcal{X}} h^{(t)}(x) p^{(t-1)}(x) dx \\
            &\leq (1+\eta)\sum_{t=1}^T \int_{x\in\mathcal{X}}h^{(t)}(x) p(x) dx+
              \frac{\text{KL}(p||p^{(0)})}{\eta}
   \end{align*}
  Pick $S = \{x\in\mathcal{X}: \sum_{t}h^{(t)}(x) < \nu T\}$ and $p$ uniform
  over $S$ to get:
  \[\gamma T\leq (1+\eta)\nu T + \frac{\log(\vol(\mathcal{X})/\vol(S))}{\eta}\]
  From this we get the bound:
  \[
    \log(\vol(S)) \leq  \log(\vol(\mathcal{X})) - \gamma \eta T + \nu \eta T  +
    \nu \eta^2 T
  \]

  Now we can get the following basic bound on $\log(p^{(T)}(x^*))$ by
  decomposing the normalizer using $S$.
  \begin{align*}
    \log(p^{(T)}(x^*)) \geq \log(1-\eta) M_T(x^*) \\
    - \log(\vol(S) + \exp( \log(1-\eta)\nu  T)(\vol(\mathcal{X})-\vol(S)))
  \end{align*}
  The rest of the proof is identical to that of Theorem 1 and we obtain:
  \begin{align*}
    &\log(p^{(T)}(x^*)) \\
    &\geq \gamma \frac{\eta}{\eta+2}\left(T -
      \frac{(\eta+1)(\eta+2)}{\gamma}M_T(x^*)\right) -\log(2\vol(\mathcal{X})).
  \end{align*}
\end{proof}

We now show that given a well-behaved starting distribution, the distribution
induced by the multiplicative weights algorithm is close to a uniform
distribution over the sublevel set.
\begin{lem}\label{lem:convapprox}
  Define $h^{*(t)}(x) = 1_{f(x) \leq \alpha^{(t)}}$ and $q^{(t)}(x) =
  \frac{1-h^{*(t)}(x)}{\int_{x\in\mathcal{X}} 1-h^{*(t)}(x)dx}$.

  Define $p^{(t)}$ and $h^{(t)}$ such that
  \begin{gather*}
    p^{(t-1)}(x) \propto 1 \quad \forall x:h^{*(t)}(x)=1,\\
    \int_{x\in\mathcal{X}} h^{(t)}(x) p^{(t-1)}(x) dx = \gamma,\\
      p^{(t)}(x) \propto p^{(t-1)}(x) (1-\eta h^{(t)}(x)).
    \end{gather*}
    and further, $\int_{x\in\mathcal{X}} 1_{h^{*(t)}(x)=0}1_{h^{(t)}(x)=1}dx = 0$
    and $\int_{x\in\mathcal{X}} 1_{h^{*(t)}(x)=1}1_{h^{(t)}(x)=0}p^{(t-1)}(x)dx
    \leq \nu$ then
    \[
      \frac{1-\eta(\gamma-\nu)}{1-\gamma-\nu} p^{(t)}(x) \geq q^{(t)}(x)
    \]
    for all $x\in\mathcal{X}$.
  \end{lem}
  \begin{proof}
    We verify the inequality on $S^*=\{x:h^{*(t)}(x)=0\}$, since $q^{(t)}(x)$ is
    zero outside $S^*$.

    For any $x \in S^*$, $q^{(t)}(x) = 1/\vol(S^*)$ and $p^{(t-1)}(x) \geq
    (1-\gamma - \nu)/\vol(S^*)$. 

    This implies for $x\in S^*$,
    \[
      p^{(t)}(x) \geq \frac{(1-\gamma-\nu)/\vol(S^*)}{\int_{x\in\mathcal{X}}
        p^{(t-1)}(x)(1-\eta h^{(t)}(x))}.
    \]

    We now upper bound the normalizer:
    \begin{align*}
      \int_{x\in\mathcal{X}} p^{(t-1)}(x)(1-\eta h^{(t)}(x))
      &\leq 1- \eta\int_{x\in\mathcal{X}} p^{(t-1)}(x)h^{(t)}(x)\\
      &\leq 1-\eta ( \gamma - \nu).
    \end{align*}

    This gives the bound that for all $x \in S^*$,
    \[\frac{1-\eta(\gamma-\nu)}{1-\gamma-\nu} p^{(t)(x)} \geq q^{(t)}(x).\]
    \end{proof}

    The next theorem gives the disagreement coefficient for linear
    classification in a log-concave density.
    \begin{thm}[Theorem 14, \cite{BalcanLo13}]\label{thm:dislin}
      Let $\mathcal{D}$ be a log-concave density in $\mathbb{R}^d$ and
      $\mathcal{H}$ the set of linear classifier in $\mathbb{R}^d$, then the
      region of disagreement over error in Definition \ref{def:disagree} is:
      \begin{equation*}
        \frac{\P_{X \sim \mathcal{D}}(X\in\disagree(B(h,r)))}{r}
        = O(d^{1/2}\log(1/r)).
      \end{equation*}
    \end{thm}
    \begin{proof}
      The statement here is identical to \cite{BalcanLo13} with the
      exception of the identity covariance constraint on $\mathcal{D}$. We
      omit this due to the equivalence of isotropic and non-isotropic cases
      as noted in Appendix C, Theorem 6.
    \end{proof}

    Finally, we combine the above results to obtain convergence rates on the
    linear optimization with linear classifier case.
    \begin{thm}\label{thm:linear-opt-classifier}
      Let $f(x) = w^\top x$ for some $w$, and $\mathcal{X}$ a convex subset of
      $\mathbb{R}^D$. Define $p^{(0)}$ as uniform on $\mathcal{X}$.

      Define the classification label $Z_i = 1_{Y_i > \text{median}(Y)}$, on
      which we train a CSS classifier over the linear hypothesis class
      $\mathcal{H}$ which we define as
      \[
        h^{(t)}(x) = \begin{cases}
         1 &\text{ if }\forall g \in \verspace_{\mathcal{H},(X,Z)}, g(x)=1\\
      0 &\text{ if }\forall g \in \verspace_{\mathcal{H},(X,Z)}, g(x)=0\\
      0 & \text{ otherwise }.
    \end{cases}
  \]

  Define the sampling distributions for each step
  \[p^{(t)}(x) \propto p^{(t-1)}(x) (1-h^{(t)}(x)/2)\]
  
  Then for any small $\nu$, a batch size of 
  \[n^{(t)} = O\left(\frac{d^{3/2}}{\nu}\log(\nu^{-1}d^{1/2}) -
      \frac{d^{1/2}}{\nu}\log(3T\delta)\right),\] is sufficient to establish the
  following bound on sampling items below any $\alpha$ level set $S_\alpha =
  \{x\in\mathcal{X}: f(x) < \alpha\}$ on the last step:
  \begin{multline*}\log\left(\int_{x\in S_\alpha}p^{(T)}(x)\right) \geq \\
    \min\bigg(\left(\frac{1}{4}-\nu\right) \frac{1}{5}T -\log(2\vol(\mathcal{X}))  + \log(\vol(S_\alpha)),\\
    -\log(4)
    \bigg)
  \end{multline*}
  with probability $1-\delta$.

\end{thm}
\begin{proof}
  We will begin by using induction to show that at each round $t$ the abstention
  rate is at most $\nu$. In the base case of $t=0$, $p^{(0)}(x)$ is log-concave
  and we can apply Theorem \ref{thm:dislin} and \ref{thm:csscover}, which implies
  $\nu = C_1 d^{1/2} \log(1/\epsilon) \epsilon$. VC dimension bounds imply there
  exists some $n = C_2 \epsilon^{-1} (d\log(1/\epsilon) + \log(1/\delta_1))$ such
  that any consistent hypothesis incurs at most $\epsilon$ population error.
  
  Inverting and solving for $n$ shows that for some $C$, 
  \[n^0 \geq C\nu^{-1} d^{1/2} [ d \log(\nu^{-1} d^{1/2}) + \log(\delta_1^{-1})]\]
  is sufficient to guarantee $\nu$-abstention.

  For each round $t>0$ assume we maintained $\nu$-abstention in all prior
  rounds. We then fulfil the condition of Lemma \ref{lem:convapprox}. To verify
  each condition: $p^{(t-1)}(x)$ is constant on sublevel sets below $\alpha^{(t)}$
  since the selective classifier to never makes false
  positives, by construction. $\eta=1/2$, $\gamma^{(t)}$ is the population
  quantile of $p^{(t-1)}$ corresponding to $\text{median}(Y^{(t-1)})$ and $\nu$
  is the abstention rate at round $t-1$.

  Lemma \ref{lem:convapprox} states that if a sample drawn $x \sim p^{(t)}$ has
  $f(x) < \text{median}(Y^{(t-1)})=\alpha^{(t-1)}$, then this sample follows the
  uniform distribution over the sublevel set $S_{\alpha^{(t-1)}}$ and that the
  probability of this event occuring is at least
  $\frac{1-\gamma^{(t)}-\nu}{1-(\gamma^{(t)}-\nu)/2}$.

  Thus, if we sample $n'$ samples from $p^{(t)}$, then with probability at least
  $1-\delta_2$, $n$ of these draws will be from a uniform distribution in the
  $\alpha^{(t)}$ sublevel set, where $\delta_2$ follows:
  \begin{multline*}
    \delta_2 = \P\left(\sum_{i=1}^{n'} 1_{X^{(t)}_i \in S_{\alpha^{(t-1)}}} \leq n\right) \\
    \leq \exp\left(- 2\left(\frac{1-\gamma^{(t)}-\nu}{1-(\gamma^{(t)}-\nu)/2}-n/n'\right)^2 n'\right)
  \end{multline*}
  Solving for $n'$, with the shorthand $\tau=\frac{1-\gamma^{(t)}-\nu}{1-(\gamma^{(t)}-\nu)/2}$ and $\delta_2 < 1$,
  \begin{align*}
    n' &\geq \frac{\sqrt{\log(\delta_2)(\log(\delta_2)- 8 \tau n)} + 4 \tau n - \log(\delta_2)}{4 \tau^2} \\
    &\geq \frac{1}{\tau}n - \frac{\log(\delta_2)}{2\tau^2}
    \end{align*}

  Since the $\alpha^{(t)}$ sublevel set is convex, we can apply Theorem
  \ref{thm:dislin} to these $n$ samples contained in the sublevel set to show
  that for all $t>0$,
  \[n   \geq C\nu^{-1} d^{1/2} [ d \log(\nu^{-1} d^{1/2}) + \log(\delta_1^{-1})]\]
  samples is sufficient to ensure $\nu$-abstention when we train a selective
  classifier on $n$ points. Learning a selective classifier over all $n'$ points
  can only decrease the abstention region, and therefore $n^t$ is sufficient to
  guarantee $\nu$ abstention.

  Combining with the bound on $n'$, with probability $1-\delta_2-\delta_1$, 
  \[
    n^{(t)} \geq \frac{1}{\tau} C\nu^{-1} d^{1/2} [ d \log(\nu^{-1} d^{1/2}) +
    \log(\delta_1^{-1})] - \frac{\log(\delta_2)}{4\tau^2}
  \]

  Next, we bound $\gamma$ from above and below using the same argument as
  Theorem \ref{thm:csscut}. At any round $t$ defining $\gamma^{(t)} =
  \int_{x\in\mathcal{X}} p^{(t-1)}(x) 1_{x<\text{median}(Y^{(t-1)})}$ by
  Hoeffding's inequality the probability $\gamma^{(t)}$ is within $[1/4,3/4]$ is:
  \[\P(1/4\leq \gamma^{(t)} \leq 3/4) \geq 1-2\exp(-n^{(t)}/8).\]
  Thus, we can ensure $\gamma^{(t)} \in [1/4,3/4]$ with probability at least
  $1-\delta_3$ in round $t$ if we have
  \[n^{(t)} \geq -8\log(\delta_3/2). \]

  Simplifying the bounds, we can ensure $\nu$ abstention across all $T$ rounds
  with probability at least $1-\delta$ given $\tau=\frac{1/4-\nu}{7/8+\nu/2}$.
  
\begin{multline*}
  n^{(t)} \geq C \frac{7/8+\nu/2}{1/4\nu-\nu^2} d^{1/2} [ d \log(\nu^{-1} d^{1/2}) - \log(3T\delta)]\\
  - \frac{\log(3T\delta)(7/8+\nu/2)^2}{4(1/4-\nu)^2}\\
  +8\log(3T\delta^{-1})
  \end{multline*}

  Collecting first order terms for $\nu$ small we have
  \[n^{(t)} = O\left(\frac{d^{3/2}}{\nu}\log(\nu^{-1}d^{1/2}) - \frac{d^{1/2}}{\nu}\log(3T\delta)\right)\]

  Finally, we can apply Lemma \ref{lem:cont-cut} to get a lower bound for all
  points $x: f(x) < \alpha^{(T)}$.
  \[\log(p^{(T)}(x)) \geq \left(\frac{1}{4}-\nu\right) \frac{1}{5}T -\log(2\vol(\mathcal{X})) .\]

  The complete bound with respect to the $\alpha$ sublevel set follows after
  checking two cases. If $\alpha^{(T)} \geq \alpha$, then the entire set
  $\{x:f(x) < \alpha\}$ follows the lower bound, and we integrate to obtain the
  first part of the bound. If $\alpha^{(T)} < \alpha$ then by construction of
  $\alpha^{(T)}$, at least $1-\gamma$ of the distribution must lie below
  $\alpha^{(T)}$ and thus we draw an element with function value less than
  $\alpha^{(T)} < \alpha$ with probability at least $\log(1-\gamma)$.
\end{proof}

\subsection{Approximating selective sampling with Gaussians}
In this section, we show that sampling from a particular Gaussian approximates
selective classification for linear classification with strongly convex losses.

First, we show that the maximum over sampled Gaussian parameter vectors is close
to the infimum over an hyperellipse.
\begin{lem}\label{lem:quadmin}
  Consider $\tq_1 \hdots \tq_B \sim N(\hat{\theta}, \Sigma)$ then for any $x$
  and $Q_\tau=\{\theta: (\theta-\hat{\theta})^\top \Sigma^{-1}
  (\theta-\hat{\theta}) \leq 2\tau\}$,
  \begin{multline*}
    P\left(\min_i x^\top\tq_i  - \inf_{\theta\in Q_\tau} x^\top\theta > \sqrt{x^\top \Sigma x}\epsilon\right) \\
    \leq (1-\Phi(\epsilon-\sqrt{2\tau}))^B.
  \end{multline*}
  Where $\Phi$ is the cumulative distribution function of the standard Gaussian.
\end{lem}
\begin{proof}
  We can whiten the space with $\Sigma^{-1/2}$ to get the following equivalent
  statement:

  \newcommand{\oltq}{\overline{\tq}}
  \newcommand{\olx}{\overline{x}}
  Define $\oltq_i = \Sigma^{-1/2}\tq_i$, then $\oltq_1 \hdots \oltq_B \sim N(0,
  I)$ and let $\olx = \Sigma^{1/2}x$,
  \begin{align*}
    &\P\left(\min_i x^\top\tq_i  - \inf_{\theta\in Q_\tau} x^\top\theta > \sqrt{x^\top \Sigma x}\epsilon\right)\\
    &\quad = P\left(\min_i \olx^{\top}\oltq_i  - \inf_{\theta:\norm{\theta}_2^2 < 2\tau} \olx^{\top}\theta > \norm{\olx}_2 \epsilon\right) \\
    &\quad = P\left(\min_i \frac{\olx^\top}{\norm{\olx}_2}\oltq_i  >  \epsilon - \sqrt{2\tau}\right)\\
    &\quad =  P\left(\frac{\olx^\top}{\norm{\olx}_2}\oltq_0  >  \epsilon - \sqrt{2\tau}\right)^B\\
    &\quad =  (1-\Phi(\epsilon-\sqrt{2\tau}))^B\\
  \end{align*}

\end{proof}

For $\epsilon=0$ and $\tau=1$, if $B \geq \log(\delta)/\log(0.95)$ then the
minimum of the $B$ samples is smaller than the infimum over the $\Sigma^{-1}$
ellipse with probability at least $1-\delta$.

Next, we show that the samples from $N(\hat{\theta},\Sigma)$ in $d$-dimensions
are contained in a ball a constant factor larger than $d\Sigma^{-1}$ with high
probability.
\begin{lem}\label{lem:ballmax}
  Let $\tq_1 \hdots \tq_B \sim N(\hat{\theta}, \Sigma)$, then:
  \begin{multline*}
    P\left( \max_i (\tq_i - \hat{\theta})^\top \Sigma^{-1}(\tq_i - \hat{\theta})^\top \geq c d \right) \\
    \leq  B\left(c \exp\left(1-c\right)\right)^{d/2}
  \end{multline*}
  for $c>1$.
\end{lem}

\begin{proof}
  Since $(\tq_i - \hat{\theta})^\top \Sigma^{-1}(\tq_i - \hat{\theta})^\top$ is
  whitened, we can define the probability in terms of chi-squared variables
  $\xi_1 \hdots \xi_B \sim \chi^2(d)$.
  \begin{align*}
    &P\left( \max_i (\tq_i - \hat{\theta})^\top \Sigma^{-1}(\tq_i - \hat{\theta})^\top \geq c d \right)\\
    &\quad = P\left( \max_i \xi_i \geq cd \right)\\
    &\quad = 1-P\left(  \xi_0 < c d\right)^B\\
    &\quad \leq 1-\left(1- \left(c \exp\left(1-c\right)\right)^{d/2}\right)^B\\
    &\quad \leq B\left(c \exp\left(1-c\right)\right)^{d/2}
  \end{align*}
  Which implies
  \begin{multline*}
    \max_i(\tq_i- \hat{\theta})^\top \Sigma^{-1} (\tq_i -\hat{\theta}) \\
    < 4\log(2B/d)+2d-4\log(\delta/2)
    \end{multline*}
    with probability
    $1-\delta$.

  \end{proof}

  Thus for $c = 1+O\left(\frac{2}{d} \left(\log(B) -
      \log(\delta)\right)\right)$, we can ensure that all $B$ samples are
  contained in the $cd\Sigma^{-1}$ ball with probability at least $1-\delta$.

  Combining the two results, we can show that for a quadratic version space, we
  can perform selective classification by sampling from a Gaussian.
  \begin{thm}
    Define the selective classifier with respect to a set $Q$ (a subset of the
    parameter space $\Theta$) as:
        \[h_Q(x)=
    \begin{cases}
      1 &\text{ if }\forall \theta \in Q, \theta^\top x > 0 \\
      0 &\text{ if }\forall \theta \in Q, \theta^\top x \leq 0\\
      \emptyset & \text{ otherwise }
    \end{cases}.\]

  Let $Q_{\tau} = \{\theta:
  (\theta-\hat{\theta})^\top\Sigma^{-1}(\theta-\hat{\theta}) < 2\tau\}$ and $\tq_1
  \hdots \tq_B \sim N(\hat{\theta}, \Sigma)$.

  Then the sampled selective classifier:

          \[\hq(x)=
    \begin{cases}
      1 &\text{ if }\forall \tq_i, x^\top \tq_i > 0 \\
      0 &\text{ if }\forall \tq_i, x^\top \tq_i  \leq 0\\
      \emptyset & \text{ otherwise }
    \end{cases}.\]

  Has the following properties.

  For any $x\in\mathcal{X}$ 
  \[P(h_{Q_{\tau}}(x) = 0 \text{ and } \hq(x) = 1 ) \leq (1-\Phi(-\sqrt{2\tau}))^B\]
  and for any $c>1$,
  \begin{multline*}
    P\left(\sum_x 1_{h_{Q_{cd}}(x)} = \emptyset < \sum_x 1_{\hq(x) = \emptyset}\right) \\
    \leq B\left(c \exp\left(1-c\right)\right)^{d/2}
    \end{multline*}
\end{thm}
\begin{proof}
  The first statement follows directly from Lemma \ref{lem:quadmin}, since
  having $h_{Q_{\tau}}(x) = 0$ and $\hq(x) = 1$ is a subset of the event that is
  upper bounded in Lemma \ref{lem:quadmin}.

  The second statement follows from Lemma \ref{lem:ballmax}. Under the
  conditions of the theorem,
  \begin{multline*}
    P\bigg( \max_i (\tq_i - \hat{\theta})^\top \Sigma^{-1}(\tq_i - \hat{\theta})^\top) \geq c d \bigg) \\
    \leq  B\left(c \exp\left(1-c\right)\right)^{d/2}.
    \end{multline*}
  If $\max_i (\tq_i - \hat{\theta})^\top \Sigma^{-1}(\tq_i - \hat{\theta})^\top
  \geq c d$, then for all $i$, $\tq_i \in Q_{cd}.$

  By construction, if $\hq(x) = \emptyset$ then, $h_{Q_{cd}(x)} = \emptyset$
  since at least one pair of $\tq_i$ must disagree, and this pair must also be
  in $Q_{cd}(x)$.
\end{proof}

Now we show that a quadratic version space contains the Bayes optimal
parameter, and is contained in a version space with low error. We first state
a result on the strong convexity of the empirical loss, which follows
immediately from Theorem 5.1.1 of \cite{tropp2015introduction}.

\begin{lem}\label{lem:strong-convex}
    Let $\{P_\theta\}_{\theta \in \Theta}$ be a parametric family over a
    compact parameter space $\Theta$.
    Define $Z \mid X \sim P_{\theta^*}$ for
    some $\theta^* \in \operatorname{int} \Theta$. 

    Let $\ell_\theta(x, z) = -\log(P(z\mid x))$ be the negative log likelihood
    of $z$, and
    \[
      L_n(\theta) = \tfrac{1}{n}\sum_{i=1}^n \ell_\theta(X_i, Z_i).
    \]
    If \begin{align*}
      & \lambda_{\text{min}}\left( \E\left[\nabla^2 \ell_{\theta}(X, Z)\right]\right) \ge
         \gamma,~\text{and} \\
      & \lambda_{\text{max}}\left( \nabla^2 \ell_{\theta}(X, Z)\right) \le
         L~\text{w.p. 1},
    \end{align*}
    then \[
      \lambda_{\text{min}}\left(\nabla^2 L_n(\theta)\right) \ge \left(1 - \sqrt{\frac{2 L
            \log(\sfrac{d}{\delta})}{n \gamma}}\right) \gamma
    \]
    with probability $1 - \delta$.
  \end{lem}

  We now define the three possible version spaces that we can consider - the version space of the low error hypotheses $VS$, version space of low-loss models $\ell$, and the quadratic approximation $Q$.
  \begin{defn}\label{defn:verspaces}
    Define the maximum likelihood estimator $\mle = \arg\min_{\theta\in\Theta} L_n(\theta)$
    
      The quadratic version space with radius $\tau$ around $L_n$ is
    \begin{multline*}
    Q_\tau = \bigg\{
        \theta: \nabla L_n(\mle)^\top (\theta - \mle) +
      \tfrac{1}{2}(\theta - \mle)^\top\\ \nabla^2 L_n(\mle) (\theta -
      \mle) \le \tau
    \bigg\}.
  \end{multline*}
  The version space of loss-loss parameters $\theta$ with radius $\tau$ is
    \begin{multline*}
    \mathcal{L}_\tau = \bigg\{
        \theta: L_n(\theta) - L_n(\mle) \le \tau
    \bigg\}.
  \end{multline*}
  The version space of a set of hypotheses $h$ with error less than $\tau$ is 
    \begin{multline*}
    VS_\tau = \bigg\{
        \theta: P\left( h_\theta(X) \not = Z \right) - P\left( h_{\mle}(X) \not = Z \right) \le \tau
    \bigg\}.
    \end{multline*}
    \end{defn}

\begin{lem}\label{lem:quadcontain}
  In the same settings as lemma \ref{lem:strong-convex}, assume $\ell_\theta(x,z)$
  majorizes the zero one loss for some hypothesis class 
  $\mathcal{H} = \{h_{\theta} : \theta \in \Theta\}$ as $\ell_\theta(x, z) \ge \ind{h_\theta(x) = z}$,
  $\|\nabla \ell_{\theta}(X, Z)\| \le R$ almost surely, and $\nabla^2 L_n(\theta)$ is $M$-Lipschitz.
  
  Then, for the version spaces given in definition \ref{defn:verspaces},
  \begin{equation}
    VS_{\tau + 2\zeta} \supseteq
    \mathcal{L}_{\tau + 2\zeta} \supseteq
    Q_{\tau + \zeta} \supseteq \mathcal{L}_{\tau},
  \end{equation}
  where $\zeta=M
  \left(\tfrac{\tau}{2\gamma}\right)^{\sfrac{3}{2}}$. Furthermore, if
  \[
    \tau =
    \frac{36 \nu_0^2}{C_f^2}\frac{2.7d + \log(\sfrac{1}{\delta})}{n},
  \]
  then $ \theta^* \in Q_\tau $ with probability $1 - \delta$ where constants $C_f$ and $\nu_0$ defined in \cite{spokoiny2012parametric} as
  \begin{align*}
    &C_f \le \sup_{\theta}
    \frac{\kl{\theta}{\theta^*}}{\|I^{-\sfrac{1}{2}}(\theta - \theta^*)\|},~\text{and}
    \\
    &\nu_0^2 \ge R\lambda_{\text{max}}(I_{\theta^*}).
  \end{align*}
\end{lem}
\begin{proof}
  The norm $\|\theta - \mle\|$ bounds the remainder term of the 2nd
  order Taylor expansion of $L_n$ from above,
  \begin{align*}
    L_n(\theta) &- L_n(\mle) \le \nabla L_n(\mle)^\top (\theta - \mle) \\
    &+ \tfrac{1}{2} (\theta - \mle)^\top \nabla^2 L_n(\mle) (\theta -
    \mle) \\ &+ \|R(\theta, \mle)\|_{\text{op}}\|\theta - \mle\|^2,
  \end{align*}
  and below,
  \begin{align*}
    L_n(\theta) &- L_n(\mle) \ge \nabla L_n(\mle)^\top (\theta - \mle) \\
    &+ \tfrac{1}{2} (\theta - \mle)^\top \nabla^2 L_n(\mle) (\theta -
    \mle) \\ &- \|R(\theta, \mle)\|_{\text{op}}\|\theta - \mle\|^2,
  \end{align*}
  with $\|R(\theta, \mle)\|_{\text{op}} \le M\|\theta - \mle\|$.

  First, we prove $\mathcal{L}_\tau \subseteq Q_{\tau + \zeta}$.

  Let $\theta \in \mathcal{L}_\tau$. Strong convexity of $L_n$ at $\mle$ implies
  \[
    \|\theta - \mle\|^2 \le \frac{2}{\gamma}(L_n(\theta) - L_n(\mle))
    \le \frac{2}{\gamma}\tau.
  \]
  Using this in the Taylor expansion above implies
  \begin{align*}
     \nabla L_n(\mle)^\top (\theta - \mle)
    + \tfrac{1}{2} (\theta - \mle)^\top \nabla^2 L_n(\mle) (\theta -
    \mle) \\
    \begin{aligned}
      &\le L_n(\theta) - L_n(\mle) + \|R(\theta,
      \mle)\|_{\text{op}}\|\theta - \mle\|^2 \\
      &\le \tau + M\left(\frac{2\tau}{\gamma}\right)^{\tfrac{3}{2}} = \tau +
      \zeta.
      \end{aligned}
  \end{align*}
  The argument to show $ Q_{\tau} \subseteq \mathcal{L}_{\tau + \zeta}$ is
  nearly identical.

  Strong convexity of $L_n$ implies strong convexity of its quadratic
  expansion with the same parameter. Let $\theta \in Q_{\tau}$. The above
  Taylor approximation shows
  \begin{align*}
    &L_n(\theta) - L_n(\mle)\le \\
    &\begin{aligned}
       &
     \nabla L_n(\mle)^\top (\theta - \mle)
    + \tfrac{1}{2} (\theta - \mle)^\top \nabla^2 L_n(\mle) (\theta -
    \mle) \\
      & + \|R(\theta, \mle)\|_{\text{op}}\|\theta - \mle\|^2 \\
      &\le \tau + M\left(\frac{2\tau}{\gamma}\right)^{\tfrac{3}{2}} = \tau +
      \zeta.
      \end{aligned}
  \end{align*}

  That $\ell(\delta, z)$ majorizes the 0-1 loss almost immediately implies
  $\mathcal{L}_{\tau} \subseteq VS_{\tau}$. Indeed, $L_n(\theta) \le \tau$
  implies $ \frac{1}{n}\sum \ind{h_\theta(X_i) \not= Z_i} \le \tau$, and so
  $h_\theta$ is in $VS_\tau$.

  That $Q_\tau$ contains $\theta^*$ with probability $1 - \delta$ follows
  from Theorem 5.2 of \cite{spokoiny2012parametric} with the stated constants.
\end{proof}

\subsection{Approximating selective sampling with the bootstrap}
\newcommand{\strconv}{\gamma}
We begin by showing the minimizer of the quadratic approximation is close to the $\tq_u$ we defined.


\providecommand{\ones}{\mathbf{1}}
\providecommand{\zeros}{\mathbf{0}}
\providecommand{\opnorm}[1]{\norm{#1}_{\rm op}}
\providecommand{\opnorms}[1]{\norms{#1}_{\rm op}}
\providecommand{\norms}[1]{\|{#1}\|}
\renewcommand{\>}{\rangle}
\newcommand{\<}{\langle}
\providecommand{\half}{\frac{1}{2}}

Define
\begin{equation*}
  L_n(\theta, u) \defeq \frac{1}{n} \sum_{i = 1}^n L_i(\theta)
  + \frac{1}{n} \sum_{i = 1}^n u_i L_i(\theta),
\end{equation*}
which is convex for all $u \ge -\ones$. The standard negative log-likelihood
is then $L_n(\theta, \zeros)$, and for simplicity in notation (and to
reflect the dependence on $u$) we let $\mle = \argmin_\theta L_n(\theta,
\zeros)$ and $\theta_u = \argmin_\theta L_n(\theta, u)$.  Recall that we
assume that $\theta \mapsto \nabla^2 L_i(\theta)$ is $M$-Lipschitz
(in operator norm) and $L_n(\theta,\zeros)$ is $\gamma$-strongly convex.

We collect a few identification and complexity results.
\begin{lem}[Rakhlin and Sridharan~\cite{RakhlinSr15}]
  \label{lem:concentration-vectors}
  Let $u_i$ be independent, mean-zero, $\sigma^2$-sub-Gaussian random
  variables. Then for any sequence of vectors $v_i$ with $\norm{v_i} \le R$
  for each $i$,
  \begin{equation*}
    \P\left(\norm{\sum_{i = 1}^n u_i v_i} \ge c R \sigma \sqrt{n} (1 + t)\right)
    \le \exp(-t^2),
  \end{equation*}
  where $c < \infty$ is a numerical (universal) constant.
\end{lem}

We then have the following guarantee.
\begin{lem}
  \label{lem:theta-u-theta-nothing}
  Let $\norm{\nabla L_i(\theta)} \leq R$ for all $i$ and $u_i$ be independent mean-zero $\sigma^2$-sub-Gaussian random variables. Then there exist numerical constants $C_1, C_2, C_3 < \infty$ such that
  for all $t \in \R$ and $n \geq \frac{d}{\gamma^2} C_2 (C_3-\log(\delta))$, we have
  \begin{equation*}
    \norm{\theta_u - \mle}
    \le C_1 \cdot \frac{R \sigma}{\sqrt{n}}
  \end{equation*}
  with probability at least $1 - \delta$.
\end{lem}
\begin{proof}
  The proof is a more or less standard exercise in localization and
  concentration~\cite{BoucheronLuMa13}. We begin by noting that $|u_i|$ is
  $O(1) \sigma^2$-sub-Gaussian, so that $\frac{1}{n} \sum_{i = 1}^n |u_i|
  \le \E[|u_i|] + t$ with probability at least $1 - \exp(-c n t^2 /
  \sigma^2)$, where $c > 0$ is a numerical constant, so we assume that
  $\sum_{i = 1}^n |u_i| \le 2 \sigma$, which occurs with probability at
  least $1 - e^{-c n}$. Non-asymptotic lower bounds on the eigenvalues of
  random matrices~\cite[Thm.~1.3]{KoltchinskiiMe13} imply that for any $t
  \ge 0$,
  \begin{equation}
    \label{eqn:big-nabla-event}
    \nabla^2 L_n(\mle; u)
    \succeq \left(\gamma - \sqrt{\frac{d}{n}} \cdot t \right) I_d
  \end{equation}
  with probability at least $1 - \exp(- c t^2 + C \log \log t )$ for
  constants $c > 0$, $C < \infty$. Then
  using the $M$-Lipschitz continuity of $\nabla^2 L$, 
  we obtain
  \begin{align*}
    \opnorm{\nabla^2 L_n(\theta; u) - \nabla^2 L_n(\mle; u)}
    & \le \frac{M}{n} \sum_{i = 1}^n |u_i| \norm{\theta - \mle} \\
    & \le 2 M \sigma \norm{\theta - \mle}.
  \end{align*}
  With this identification, for all $\theta$ satisfying $\norm{\theta -
    \mle} \le \frac{\epsilon}{2 M \sigma}$, we have
  \begin{align*}
    L_n(\theta, u)
    & \ge L_n(\mle, u) + \<\nabla L_n(\mle, u),
    \theta - \mle\> \\
    & \qquad ~ + \frac{\gamma - \epsilon
      - t \sqrt{d/n}}{2} \norm{\theta - \mle}^2
  \end{align*}
  with probability at least $1 - \exp(-c t^2 + C \log \log t)$.
  Now, using Lem~\ref{lem:concentration-vectors},
  we obtain that $\norm{\nabla L_n(\mle; u)} \le
  \frac{c R \sigma}{\sqrt{n}} (1 + t)$ with probability at least
  $1 - e^{-t^2}$, so that we have
  \begin{align*}
    L_n(\theta, u)
    & \ge L_n(\mle, u)
    - \frac{C R \sigma}{\sqrt{n}} (1 + t) \norm{\mle - \theta} \\
    & \qquad ~~ + \frac{\gamma - \epsilon - t \sqrt{d/n}}{2}
    \norm{\mle - \theta}^2
  \end{align*}
  with probability at least
  $1 - 2 e^{-c t^2 + C \log \log t}$, where $0 < c, C < \infty$ are
  numerical constants.

  Now fix $\epsilon$. Solving the above quadratic
  in $\norm{\mle - \theta}$, we have that
  \begin{equation*}
    \norm{\theta - \mle} >
    \frac{2 C R \sigma}{
      \sqrt{n}(\gamma - \epsilon) - t\sqrt{d}}
  \end{equation*}
  implies that $L_n(\theta, u) > L_n(\mle, u)$.
  If we assume for simplicity $\epsilon = \gamma / 2$, then
  we have that
  \begin{equation*}
    \frac{\gamma}{2}
    \ge \norm{\theta - \mle} >
    \frac{4 C R \sigma}{
      \sqrt{n}\gamma - 2t\sqrt{d}}
  \end{equation*}
  implies that $L_n(\theta, u) > L_n(\mle, u)$, or (by convexity)
  that
  any minimizer $\theta_u$ of $L_n(\theta, u)$ must satisfy
  $\norm{\theta_u - \mle}
  \le \frac{4 C R \sigma}{\sqrt{n} \gamma - 2t \sqrt{d}}$.

  The resulting bound states that there exists numerical constants $c > 0$, $C < \infty$ such that for all $t \in \R$, we have
  \[ \norm{\theta_u  - \mle} \leq C \frac{R\sigma}{\sqrt{n}} \cdot \frac{1}{\gamma-2 t \sqrt{d/n}}\]
  with probability at least $1-C e^{-ct^2+\log \log(t)}$.  

  Upper bounding $\log(\log(t)) \leq \log(t)$ and solving for $t$, for any $t \geq C_2 \sqrt{C_3-\log(\delta)}$ there exists $C_2, C_3 < \infty$ such that $1-C e^{-ct^2+\log(\log(t))} > 1-\delta$.

  Finally, if $n \geq \frac{d}\gamma^2 C_2^2 (C_3-\log(\delta))$, then we have an upper bound,
  \[ \norm{\theta_u  - \mle} \leq 2C \frac{R\sigma}{\sqrt{n}} \]
  with probability at least $1-\delta$.

\end{proof}

Define the minimizer for a quadratic approximation to the multiplier bootstrap loss as
\begin{equation*}
    \tq_u \defeq \mle - \nabla^2 L_n(\mle)^{-1} 
    \frac{1}{n} \sum_{i = 1}^n u_i \nabla L_i(\mle).
  \end{equation*}

We now show that for subgaussian $u$, $\tq_u \approx \theta_u$.
\begin{lem}\label{lem:quadclose}
  Let $\norm{\nabla L_i(\theta)} \leq R$, $\opnorm{\nabla^2 L_i(\mle)} \le S$  for all $i$, and $u_i$ be independent mean-zero $\sigma^2$-sub-Gaussian random variables.
  Then there exist numerical constants $C, C_1, C_2, C_3 < \infty$ such that for all $\delta>0$ and $n \geq \frac{d}{\gamma^2} C_2 (C_3-\log(\delta/4))$,
  \begin{equation*}
    \norm{\theta_u - \tq_u}
    \le C\frac{R^2\sigma^2}{n}\left(2M + \frac{S\sqrt{\log(4d/\delta)}}{R}\right)
  \end{equation*}
  with probability at least $1-\delta$.
\end{lem}
\begin{proof}
  By definition, we have
  \begin{equation*}
    \tq_u \defeq \mle - \nabla^2 L_n(\mle)^{-1} 
    \frac{1}{n} \sum_{i = 1}^n u_i \nabla L_i(\mle).
  \end{equation*}
  Now, consider $\theta_u$. By Taylor's theorem and the $M$-Lipschitz
  continuity of $\nabla^2 L_i$, we have that for matrices $E_i : \R^d \to
  \R^{d \times d}$ with $\opnorm{E_i(\theta)} \le M \norm{\theta -
    \mle}$ that for all $\theta$ near $\mle$, we have
  \begin{align*}
    \lefteqn{\sum_{i = 1}^n (1 + u_i) \nabla L_i(\theta)} \\
    & = n \nabla^2 L_n(\mle) (\theta - \mle)
    + \sum_{i = 1}^n u_i \nabla L_i(\mle) \\
    & ~ + \sum_{i = 1}^n E_i(\theta) (\theta - \mle)
    + \sum_{i = 1}^n u_i \left[
      \nabla^2 L_i(\mle) + E_i(\theta) \right] (\theta_u - \mle).
  \end{align*}
  That is, defining the random matrix
  \begin{equation*}
    E \defeq
    \frac{1}{n} \sum_{i = 1}^n \left[(1 + u_i) E_i(\theta_u)
      + u_i \nabla^2 L_i(\mle)\right]
  \end{equation*}
  and noting that
  $\zeros = \frac{1}{n} \sum_{i = 1}^n (1 + u_i) \nabla L_i(\theta_u)$,
  we have
  \begin{equation*}
    -\frac{1}{n} \sum_{i = 1}^n u_i \nabla L_i(\mle)
    = \left(\nabla^2 L_n(\mle) + E \right) (\theta_u - \mle).
  \end{equation*}
  Using standard matrix concentration
  inequalities~\cite{tropp2015introduction}
  yields that
  \begin{equation*}
    \P\left(\opnorm{\sum_{i = 1}^n u_i \nabla^2 L_i(\mle)}
    \ge t \right)
    \le 2d \exp\left(-\frac{c t^2 n}{\sigma^2 S^2}\right)
  \end{equation*}
  for a numerical constant $c > 0$. Using that
  $\opnorm{E_i(\theta)} \le M \norm{\theta - \mle}$, we obtain that
  with probability at least $1 - 2d \exp(-\frac{c t^2 n}{\sigma^2 S^2})$, thus
  we have $\opnorm{E}
  \le 2 M \norm{\theta_u - \theta} + \frac{S\sigma\sqrt{\log(2d/\delta)}}{\sqrt{2}}$
  with probability at least $1 - \delta$.
  Using Lem~\ref{lem:theta-u-theta-nothing}, we have that for all $n \geq \frac{d}{\gamma^2} C_2 (C_3-\log(\delta/4))$,
  \begin{equation*}
    \theta_u
    = \mle - \left(\nabla^2 L_n(\mle)
    + E \right)^{-1} \frac{1}{n} \sum_{i = 1}^n u_i \nabla L_i(\mle)
  \end{equation*}
  with probability at least $1 - \delta$,
  where $\opnorm{E} \le 2 C M \frac{R\sigma}{\sqrt{n}} +  \frac{S\sigma\sqrt{\log(4d/\delta)}}{\sqrt{2n}}.$

  Now we use the fact that if $\opnorm{E A^{-1}} < 1$ for symmetric matrices
  $A, E$, then $(A + E)^{-1} = A^{-1} \sum_{i = 0}^\infty (-1)^i (E
  A^{-1})^i$.
  Using our bounds on $2 M \epsilon_n + t < \frac{1}{2 \gamma}$,
  we have
  \begin{equation*}
    (\nabla^2 L_n(\mle) + E)^{-1}
    = \nabla^2 L_n(\mle)^{-1}
    + \tilde{E},
  \end{equation*}
  where the matrix $\tilde{E}$ satisfies
  $\opnorms{\tilde{E}} \le 2 \opnorm{E}$.
  Applying Lem~\ref{lem:concentration-vectors} gives the
  result.
\end{proof}

We now define the constant terms of the bootstrap error into a function $\psi$ to simplify notation for the remaining sections.
\begin{defn}[Bootstrap error constants]\label{def:bootconst}
    Under the conditions of Lemma \ref{lem:quadclose}, define
    \[
    \psi(n,d,\delta) = C \frac{R^2}{n} \left(2M + \frac{S\sqrt{\log(4d/\delta)}}{R}\right).
\]
\end{defn}
\newcommand{\scov}{\sum_{i=1}^n \nabla L_i(\theta) \nabla L_i(\theta)^\top}
Now note that if $u_i \sim U(-1,1)$ then $\E[\tq_u] = \mle$ and $\textbf{Cov}[\tq_u] = \nabla^2
L_n(\mle, 0)^{-1}\scov\nabla^2 L_n(\mle, 0)^{-1}\sigma^2/n$, which is close
to the distribution used in the Gaussian sampling section before.

We now show that for $\tc_u \defeq \rho(\theta_u - \mle)+\mle$, $\norm{\tc_u - \tq_u}_2$ is small enough to make the quadratic approximation hold
for the bootstrap samples.
\begin{lem}\label{lem:quadmin-boot}

  Under the conditions of Lemma \ref{lem:theta-u-theta-nothing} and \ref{lem:quadclose}, define $\tc_{u_1}
  \hdots \tc_{u_B}$ where $\tc_{u_i} \defeq \sigma(\theta_{u_i} - \theta_n)+\theta_n$ and $u_{ij} \sim U(-1,1)$.

  For any $x$ and $Q_{\tau/n}=\{\theta: (\theta-\mle)^\top \nabla^2L_n(\theta, 0)(\theta-\mle) \leq \tau/n\}$. 
  \[\min_i x^\top \theta_{u_i}^\circ - \inf_{\theta\in Q_{\tau/n}} x^\top \theta \leq 0\]
  with probability at least $1-\delta$ whenever
  \begin{align*}
    B &\geq \log(\delta/3)/\log(1-\Phi(1)/2)\\
    n &\geq \max\left(\frac{8R^2}{3S} \log(3d/\delta), 4C^2 S^4R^2\gamma^{-2}/\Phi(1)^{-2}\right)\\
    \sigma &\geq \sqrt{2\tau}\frac{\left(1+ \frac{\sqrt{8}\gamma^{-7/2}S^2R}{\sqrt{3}(1 + S)}\sqrt{\log(2d/\delta)}\right)}{\left(1-\psi(n,d,\delta) S \gamma^{-1/2} \right)}
  \end{align*}
  where $\psi$ is defined in Definition \ref{def:bootconst}.
\end{lem}
\begin{proof}
  The proof proceeds in two parts: in the first, we bound the mismatch between
  the covariance matrix of $\tq_u$ and $\nabla^2 L_n(\mle)^{-1}$. In the second
  part, we bound the gap between $\tc$ and $\tq$.

  \newcommand{\scinv}{\Sigma}
  \newcommand{\scsqrt}{\Sigma^{-1/2}}
  \newcommand{\scinvsqrt}{\Sigma^{1/2}}
  \newcommand{\ox}{\overline{x}}
  \newcommand{\ot}{\overline{\theta}^\circ}
  Define the covariance matrix $\scinv = \nabla^2 L_n(\theta)^{-1} \scov
  \nabla^2 L_n(\theta)^{-1}/n$, the whitened samples $\ox = \scinvsqrt x$, and
  whitened parameters $\ot_u = \scsqrt (\tc_u-\mle)$.

  By whitening and applying the operator norm bound we obtain the upper bound
  \begin{align*}
    &\min_i x^\top \theta_{u_i}^\circ - \inf_{\theta\in Q_{\tau/n}} x^\top \theta \\
    &= \min_i \ox^\top \ot_{u_i} - \inf_{\norm{\theta}_2^2 \in 2\tau} \ox^\top \scsqrt \nabla^2 L_n(\theta)^{-1/2}\theta/\sqrt{n}\\
    &\leq \norm{\ox}_2 \bigg( \min_i \frac{\ox^\top}{\norm{\ox}_2} \ot_{u_i}\\
    &\qquad +  \sqrt{2\tau}\norm{I - \scsqrt \nabla^2 L_n(\theta)^{-1/2}/\sqrt{n}}_{op}\bigg).
  \end{align*}

  The second term is the error from the quadratic approximation, and we can bound this via the Ando-Hemmen inequality (Proposition 3.2, \cite{van1980inequality})
  which states $\norm{A^{1/2} - B^{1/2}} \leq \frac{1}{\norm{A^{-1/2}}_{op}^{-1}
    + \norm{B^{-1/2}}_{op}^{-1}} \norm{A - B}_{op}$, which gives
  \begin{align*}
    &\norm{I - \scsqrt \nabla^2 L_n(\theta)^{-1/2}/\sqrt{n}}_{op} \\
    &\leq \frac{n^{-1/2}\norm{\scsqrt \sqrt{n}}_{op}\norm{\scinv n - \nabla^2 L_n(\theta)^{-1}}_{op}}{\norm{\scsqrt\sqrt{n}}_{op}^{-1} + \norm{\nabla^2 L_n(\theta)^{1/2}}^{-1}} \\
    &\leq \frac{\sqrt{n}\gamma^{-7/2}}{S^{-3/2} + S^{-1/2}} \norm{\scov- \nabla^2 L_n(\theta)}_{op}.
  \end{align*}
  Finally, applying the Matrix Bernstein inequality (Thm 6.1.1, \cite{tropp2015introduction}), 
  \begin{multline*}\norm{I - \scsqrt \nabla^2 L_n(\theta)^{-1/2}}_{op}  \\
    \leq \frac{\sqrt{8}\gamma^{-7/2}S^2R}{\sqrt{3}(1 + S)}\sqrt{\log(d/\delta)}
  \end{multline*}
  with probability at least $1-\delta$, when $n\geq \frac{8R^2}{3S} \log(d/\delta)$.

  For the other parts of the bound let $\overline{\tq}_{u_i} =
  \scinvsqrt(\tq_{u_i}-\mle)$ which is a $d$-dimensional unit Gaussian with
  variance $\sigma^2$. By Lemma \ref{lem:theta-u-theta-nothing}, 
  \begin{align*}
    &\min_i x^\top \theta_{u_i}^\circ - \inf_{\theta\in Q_{\tau/n}} x^\top \theta \\
    &\leq \norm{\ox}_2 \bigg( \min_i \frac{\ox^\top}{\norm{\ox}_2} \overline{\tq}_{u_i} \\
    &\qquad +  \psi(n,d,\delta/3) S \gamma^{-1/2} \sigma \\
    &\qquad +  \sqrt{2\tau}\left(1+ \frac{\sqrt{8}\gamma^{-7/2}S^2R}{\sqrt{3}(1 + S)}\sqrt{\log(3d/\delta)} \right)\\
  \end{align*}
  with probability $1-2\delta/3$. Define
  \[
    \sigma= \sqrt{2\tau}\frac{\left(1+ \frac{\sqrt{8}\gamma^{-7/2}S^2R}{\sqrt{3}(1 + S)}\sqrt{\log(2d/\delta)}\right)}{1-\psi(n,d,t) S \gamma^{-1/2}}.
  \]
  Since $\ox^\top \overline{\tq}$ is a bounded i.i.d sum of $n$ gradient terms, by the Berry-Esseen theorem there exists a universal constant $C$ such that
  \begin{multline*}
    P\left(\min_i x^\top \theta_{u_i}^\circ - \inf_{\theta \in Q_{\tau/n}} x^\top\theta \leq 0 \right) \\
    \leq
     \left(1-\Phi(\tau) + \frac{C S^2 R}{\gamma\sqrt{n}} \right)^B.
   \end{multline*}
   
   If $B \geq \log(\delta/3)/\log(1-\Phi(1)/2)$ and $n \geq 4C^2 S^4R^2\gamma^{-2}/\Phi(1)^{-2}$, we have that:
  \[\min_i x^\top \theta_{u_i}^\circ - \inf_{\theta\in Q_{\tau/n}} x^\top \theta \leq 0\]
  with probability at least $1-\delta$.

\end{proof}

Thus to satisfy the same bounds as the bootstrap, we must expand the ellipse
under consideration by a small multiplicative factor.

We first define the bootstrap based selective classifier.
\begin{defn}\label{defn:bootclass}
    Let $\{P_\theta\}_{\theta \in \Theta}$ be a parametric family over a
  compact parameter space $\Theta$. Define $ Z \mid X \sim P_{\theta^*}$ for
  some $\theta^* \in \operatorname{int} \Theta$. 

  Let $\ell_\theta(x, z) = -\log(P(z\mid x))$ be the negative log likelihood
  of $z$, and assume that $\|\nabla \ell_{\theta}(X, Z)\| \le R$
  and $\norm{\nabla^2\ell(X,Z)}_{op} \leq S$ almost surely and 
  this majorizes the   0-1 loss of a linear classifier, $\ell_\theta(x, z) \ge \ind{(2z-1) x^\top \theta
    < 0 }$.
  
  Define the weighted sample negative log likelihood
  \[
    L_n(\theta, u) \defeq \tfrac{1}{n} \sum_{i=1}^n (1+u_i) \ell_\theta(X_i, Z_i),
  \]
  and we assume $L_n(\theta, 0)$ to be $\gamma$-strongly convex and
   $\nabla^2 L_n(\theta, 0)$ to be $M$-Lipschitz.

   Given a scaling constant   $\sigma > 0$, define the $B$ bootstrapped estimators $\tc_{u_i} = \sigma(\theta_u - \theta_n) + \theta_n$ with $u_{i1} \hdots u_{in} \sim U(-1,1)$.

   The bootstrap selective classifier is defined as
  \[h^\circ_{u}(x) \defeq \begin{cases}
    1 &\text{ if }\forall i, x^\top \tc_{u_i} > 0 \\
    0 &\text{ if }\forall i, x^\top \tc_{u_i} \leq 0\\
    \emptyset & \text{ otherwise }
  \end{cases}.\]
  \end{defn}

Combining the above results in the following bootstrap based selective classification bound.
\begin{thm}\label{thm:bootlin}

  The bootstrap selective classifier in definition \ref{defn:bootclass} fulfils
  \[\int_{x\in \reals^d}\ind{h^\circ_u(x) = 0 \text{ and }x^\top \theta^* \ge 0}p(x)dx
    \leq \epsilon,\]
  \[\int_{x\in \reals^d}\ind{h^\circ_u(x)=\emptyset}p(x)dx \leq \epsilon \discoeff_h,\]
  and for any $x$,
  \[P(x^\top \theta^* \leq 0 \text{ and } h^\circ_{u}(x) = 1) < \delta\]
  with probability at least $1-\delta$ whenever
  \[B \geq \log(\delta/3)/\log(\Phi(1)),\]
  \begin{multline*}
    \sigma=\sqrt{\frac{72 \nu_0^2}{C_f^2} (2.7d +
      \log\left(\sfrac{3}{\delta}\right))} \frac{\bigg(1+\frac{\sqrt{8}\gamma^{-7/2}S^2R}{\sqrt{3}(1 + S)}\sqrt{\log(2d/\delta)}\bigg)}{1-\psi(n, d, \delta/2) S \gamma^{-1/2}}\\
    =O(d^{1/2} + \log(1/\delta)^{1/2}),
  \end{multline*}
  \begin{multline*}
    \epsilon \geq \bigg(2\frac{\sigma^2\nu}{n}      + M
      \left(\frac{\sigma^2\nu}{2n\gamma}\right)^{3/2}\bigg)=O(\sigma^2/n) ,
  \end{multline*}
  \[\nu =  C\frac{R^2S}{n}\]
  and
  \[n  \geq 2\log(2d/\delta)S/\gamma^2.\]
The constants $C_f$ and $\nu_0$ are as defined in Lemma \ref{lem:quadcontain}.
\end{thm}
\begin{proof}
  Consider the first inequality on the classification of $x$. The parameter
  $\sigma$ is set such that according to Lemma \ref{lem:quadmin-boot}, we will
  achieve the infimum over $Q_\tau$ with probability $1-\delta/2$ with the given
  $n$ and $B$ for any $\tau \leq \frac{36 \nu_0^2}{C_f^2}\frac{2.7d + \log(\sfrac{1}{\delta})}{n}$.
  By Lemma \ref{lem:quadcontain}, $\theta^*\in Q_\tau$ for $\tau =
  \frac{36 \nu_0^2}{C_f^2}\frac{2.7d + \log(\sfrac{1}{\delta})}{n}$ and thus
  achieving the infimum over $Q_\tau$ is equivalent to consistency with $x^\top
  \theta^*$.

  For the second inequality on abstention, we can begin by applying Lemma
  \ref{lem:theta-u-theta-nothing} to the bootstrapped samples $\tc_{u_i}$. Lemma
  \ref{lem:theta-u-theta-nothing} directly gives the result that with probability at
  least $1-\delta/2$, $\tc_{u_1} \hdots \tc_{u_B} \in Q_{\tau'}$ where $\tau' =
  \sigma^2 \nu$. 

  Next, we can apply Lemma \ref{lem:quadcontain} to show that $Q_\tau$ is strictly contained in the
  version space of (linear) hypotheses which incur at most $\epsilon =
  2\tau^2/n+M(\tau^3n^{-3/2}\gamma^{-3/2}2^{-3/2})$ error. Combining with the
  $\tau'$ above implies the error rate, and the definition of the disagreement
  coefficient implies the abstention bound.

  For the misclassification result, we note that since all hypotheses are contained within
  a version space with $\epsilon$ error, the consensus classifier $h$ can make at most
  $\epsilon$ error.
    \end{proof}

  Interpreting these results in the context of our optimization algorithm,
\begin{thm}
  \label{thm:csscutboot}
  Let $\{P_\theta\}_{\theta \in \Theta}$ be a parametric family over a
  compact parameter space $\Theta$, and let $\mc{H} = \{ \theta^\top x \ge 0 :
  \theta \in \Theta\}$, and assume that $\mc{H}$ contains the sublevel sets of $f$.
  
  Given $T$ rounds of algorithm \ref{alg:cutplane1} with $h$ defined as in Theorem \ref{thm:bootlin},
  we have that for $x^* = \arg\min_x f(x)$,
  \begin{multline*}
    \log(p^{(T)}(x^*)) \geq \min\bigg((\xi-(\discoeff_h + 1)\epsilon) \frac{\eta}{\eta+2}T -\log(2|\mathcal{X}|) ,\\
    \log(1-\gamma) \bigg).
    \end{multline*}
  This holds with probability $1-\delta$ as long as the conditions in Theorem
  \ref{thm:bootlin} are satisfied with probability $1 - \delta/T$, and
  \[
    \begin{aligned}
      n \ge \max \{ \frac{2S}{\gamma^2} (\log(2d/\delta) + \log(2T)) ,\\
        \frac{1}{2(\xi
        - 0.5)^2}(\log(\delta^{-1}) + \log(2T)) \}.
    \end{aligned}
  \]
\end{thm}
\begin{proof}
  The proof follows almost identically to Theorem~\ref{thm:csscut}, given the
  error bounds in Theorem~\ref{thm:bootlin} hold. However, in
  Theorem~\ref{thm:bootlin}, a small number of misclassification mistakes can be
  made on elements besides $x^*$. This could reduce the fraction of the feasible space
  removed at each step, so we will provide the guarantee here for completeness.

  Without making any mistakes or abstentions, $h\sups{t}$ would remove  $\gamma$ fraction
  of the current feasible space. As before, abstention might increase this by 
  $\discoeff_h \epsilon$. Beyond this, misclassifications could prevent us from removing an additional  $\epsilon$ fraction of the feasible space, giving the bound of
  \begin{equation}\label{eq:holdout}
    \sum_{x\in \mathcal{X}} h^{(t)}(x) p^{(t-1)}(x) \geq \gamma - (1 + \discoeff_h)\epsilon.
  \end{equation}

  Beyond this change, the proof proceeds as in Theorem~\ref{thm:csscut}, with
  the choice of $n$ required in Theorem~\ref{thm:bootlin}, instead of the usual
  requirements for exact CSS.
\end{proof}

\subsection{Non-realizable agnostic selective classification}\label{sec:agnostic}

We begin by re-introducing the notation and conditions of agnostic selective
classification, as defined by El-Yaniv \cite{wiener2011agnostic}.

\begin{defn}
  Let loss class $\mathcal{F}$ be defined as:
  \[\mathcal{F} = \{\ell(h(x),z) - \ell(h^*(x),z):h\in\mathcal{H}\}\]
  This $\mathcal{F}$ is defined as $(\beta,B)$-Bernstein with respect to $\P$ if
  for all $f\in\mathcal{F}$ and some $0<\beta\leq 1$ and $B\geq 1$ if
  \[\EE f^2 \leq B(\EE f)^\beta.\]
\end{defn}

Define the empirical loss bound:
\[\sigma(n, \delta, d) = 2 \sqrt{2 \frac{d \log(2 n e/d) + \log(2/\delta)}{n}}.\]
This is a bound on the deviation in loss for a classifier with VC dimension $d$
learnt with $n$ samples with probability at least $1-\delta$.

Now define the empirical loss minimizer
\[\hat{h} = \arg\min_{h \in \mathcal{H}} \{\sum_{i=1}^n\ell(h(x_i),z_i)\}\]
and the excess loss incurred by disagreeing with the ERM on a particular point:
\begin{align*}
  \Delta(x) &= \min_{h\in \mathcal{H}} \big\{\EE[\ell(h(x),z)] | h(x) = -\text{sign}(\hat{h}(x))\big\}\\
  &\qquad -\EE[\ell(\hat{h}(x),z)]
  \end{align*}

Define the agnostic selective classifier:
\begin{equation}\label{eq:agsel}
  h_{n,\delta,d}(x)=
    \begin{cases}
      1 &\text{ if }\Delta(x) < \sigma(n,\delta,d) \text{ and }\hat{h}(x) = 1\\
      0 &\text{ if }\Delta(x) < \sigma(n,\delta,d) \text{ and }\hat{h}(x) = 0\\
      \emptyset & \text{ otherwise }
    \end{cases}.
    \end{equation}

The main theorem of agnostic selective classification is the following:
\begin{thm}\label{thm:agnosticselective}
  Assume $\mathcal{H}$ has VC dimension $V$, disagreement coefficient $\discoeff_h$
  and $\mathcal{F}$ is $(\beta, B)$-Bernstein with respect to $\P$.

  With probability at least $1-\delta$,
  \[\P(h_{n,\delta,V}(x) = \emptyset) \leq B \discoeff_h(4\sigma(n,\delta/4,V))^\beta.\]
  With a performance bound
  \[\E[\ell(h_{n,\delta,V}(x),z) -\ell(h^*(x),z)  | h_{n,\delta,V}(x)\neq\emptyset]=0,\]
  where
  \[h^*(x) = \arg\min_{h\in\mathcal{H}} \{\EE[\ell(h(x),z)]\}.\]  
  \end{thm}

\subsection{Characterizing performance from misclassification rates for
  classifiers}
Currently we have the requirement that an oracle return classifiers which
control $M_T(x^*)$ for the true optimum $x^*$. However, we are often interested
in finding one of the $K$ best elements in $\mathcal{X}$. This slight relaxation
allows us to obtain bounds that rely on controlling 0-1 losses in $\mathcal{X}$.

Combining our earlier selective classification based optimization bound with Theorem \ref{thm:agnosticselective} gives the following rate:
  \begin{thm}\label{thm:cssagn}

    Let $\mathcal{H}$ be a hypothesis class with VC dimension $V$ and disagreement coefficient $\discoeff_h$, and let $\mathcal{F}$ be a $(\beta, B)$-Bernstein class such that for each $p^{(t)}$ the population loss minimizer correctly classifies $x^*\defeq \arg\min_{x\in\mathcal{X}} f(x)$.

    There exists a numerical constant $C < \infty$
  such that for all $\delta \in [0, 1]$, 
  and $\gamma \in (B \discoeff_h(4\sigma(n,\delta/8,V))^\beta, \half)$,
  \[n \geq \frac{1}{2(\gamma-0.5)^2}(\log(\delta^{-1})+\log(2T)),\]
  with probability at least $1 - \delta$
  \begin{multline*}
    \log(p^{(T)}(x^*)) \geq \min\Big\{(\gamma-B \discoeff_h(4\sigma(n,\delta/8,V))^\beta)\frac{\eta}{\eta+2}T \\
    -\log(2|\mathcal{X}|), \log(1-\gamma) \Big\}
  \end{multline*}
  after $T$ rounds of Algorithm~\ref{alg:cutplane1} where the classifier is replaced by the agnostic selective classifier \eqref{eq:agsel}.
\end{thm}
\begin{proof}
  The proof follows directly from the proof of Theorem \ref{thm:csscut}, with the added observation that the agnostic selective classifier \eqref{eq:agsel} must always agree with the population loss minimizer, or else abstain since $\sigma(n,\delta, V)$ is a bound on the excess loss of the empirical minimizer which holds with probability at least $1-\delta$. Applying the implied abstention rate from Theorem \ref{thm:agnosticselective} then completes the proof. 
\end{proof}

\end{document}